\documentclass{article}
\usepackage{iclr2027_conference,times}
\usepackage{url}

\usepackage{algorithm}
\usepackage{algorithmicx}
\usepackage[noend]{algpseudocode}
\usepackage{amsfonts}
\usepackage{amsmath}
\usepackage{amssymb}
\usepackage{amsthm}
\usepackage{bbm}
\usepackage{bm}
\usepackage{caption}
\usepackage{color}
\usepackage{colortbl}
\usepackage{dirtytalk}
\usepackage{dsfont}
\usepackage{enumitem}
\usepackage{graphicx}
\usepackage{listings}
\usepackage{mathtools}
\usepackage{soul}
\usepackage{subcaption}
\usepackage{tcolorbox}
\usepackage{times}
\usepackage{wrapfig}
\usepackage{xspace}

\usepackage[dvipsnames]{xcolor}
\usepackage[bookmarks=false]{hyperref}
\hypersetup{
  linktocpage=true,
  colorlinks=true,
  citecolor=Green,
  linkcolor=Green,
  urlcolor=Green
}
\usepackage[capitalize,noabbrev]{cleveref}

\usepackage{thmtools}
\declaretheorem[name=Theorem,refname={Theorem,Theorems},Refname={Theorem,Theorems}]{theorem}

\declaretheorem[name=Proposition,refname={Proposition,Propositions},Refname={Proposition,Propositions},sibling=theorem]{proposition}

\newcommand{\cC}{\mathcal{C}}

\newcommand{\cL}{\mathcal{L}}

\newcommand{\cN}{\mathcal{N}}

\newcommand{\normw}[2]{\|#1\|_{#2}}

\mathchardef\mhyphen="2D

\newcommand{\norm}[1]{\left\lVert #1 \right\rVert_2}
\newcommand{\E}{\mathbb{E}}
\newcommand{\Var}{\mathrm{Var}}

\newcommand{\pold}{p_{\mathrm{old}}}

\newcommand{\qeta}{q_\eta}

\newcommand{\vtheta}{v_\theta}
\newcommand{\vold}{v_{\mathrm{old}}}

\newcommand{\diffusionnft}{\ensuremath{\color{Green}\tt DiffusionNFT}\xspace}
\newcommand{\flowgrpo}{\ensuremath{\color{Green}\tt Flow\mhyphen GRPO}\xspace}
\newcommand{\advantageflow}{\ensuremath{\color{Green}\tt AdvantageFlow}\xspace}
\newcommand{\aflow}[1]{\ensuremath{\color{Green}\tt AdvantageFlow(#1)}\xspace}
\newcommand{\awm}{\ensuremath{\color{Green}\tt AWM}\xspace}

\newcommand{\old}{\mathrm{old}}

\newcommand{\ftheta}{f_\theta}
\newcommand{\fold}{f_{\mathrm{old}}}
\newcommand{\fref}{f_{\mathrm{ref}}}
\newcommand{\lossAFM}{\mathcal{L}_{\mathrm{AFM}}}

\newcommand{\ellNFT}{\ell_{\mathrm{NFT}}}
\newcommand{\gammaNFT}{\gamma_{\mathrm{NFT}}}

\newcommand{\dif}{\mathop{}\!\mathrm{d}}
\newcommand{\inner}[2]{\left\langle #1, #2 \right\rangle}
\DeclareMathOperator{\clip}{clip}

\title{AdvantageFlow: Regularized Advantage-Weighted RL in Flow Models}

\author{Branislav Kveton,
  Anup Rao,
  Subhojyoti Mukherjee,
  Krishna Kumar Singh,
  Viet Dac Lai \\
  Adobe Research \\
  \texttt{\{kveton,anuprao,subhomuk,krishsin,daclai\}@adobe.com}
}

\iclrfinalcopy 

\begin{document}

\maketitle

\begin{abstract}
We present AdvantageFlow, a forward-process reinforcement learning (RL) algorithm for rectified flow models. The algorithm minimizes an advantage-weighted prediction loss, which maximizes reward, regularized by the rollout policy, which convexifies the objective and makes its optimization stable. Our objective can be viewed as fitting a local reward-improving target distribution. The rollout regularization arises as a variance reduction step. We evaluate AdvantageFlow empirically on text-to-image generation with Stable Diffusion 3.5 Medium and FLUX.1, and compare it to both forward- and reverse-process RL algorithms.
\end{abstract}

\begin{figure}[h!]
  \centering
  {\tiny
  \begin{tabular}{@{\ }rc@{\ }c@{\ }c@{\ }c@{\ }c@{\ }c@{\ }}
    & umbrella &
    toothbrush and &
    three apples &
    dog above &
    red giraffe &
    orange skateboard and \\
    & &
    carrot &
    &
    cow &
    &
    pink bowl \\
    \rotatebox{90}{\hspace{0.1in}Base model + cfg} &
    \includegraphics[width=0.8in]{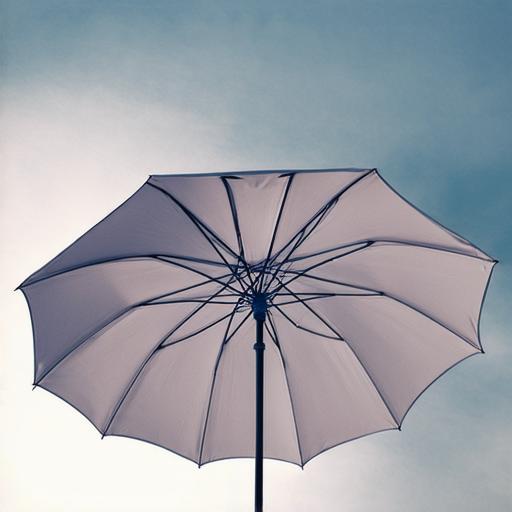} &
    \includegraphics[width=0.8in]{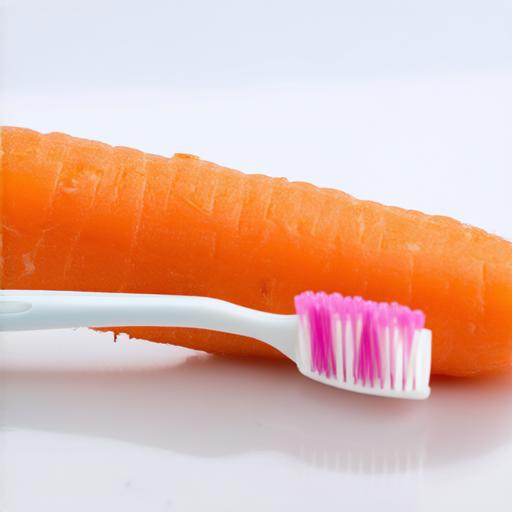} &
    \includegraphics[width=0.8in]{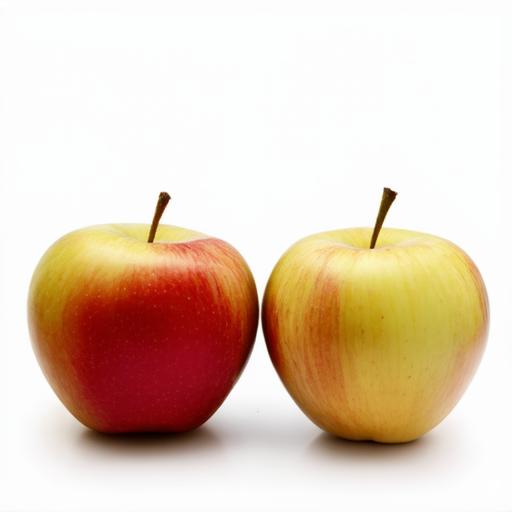} &
    \includegraphics[width=0.8in]{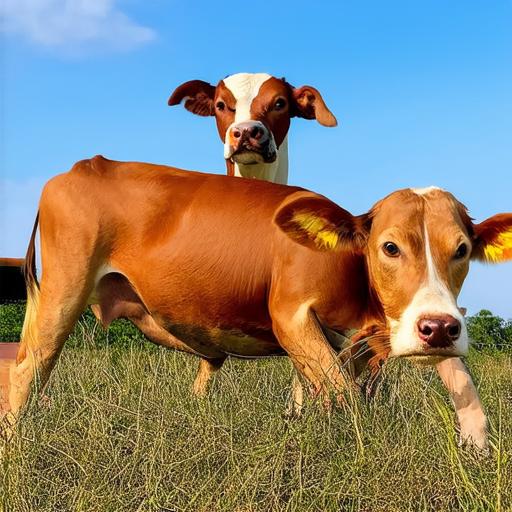} &
    \includegraphics[width=0.8in]{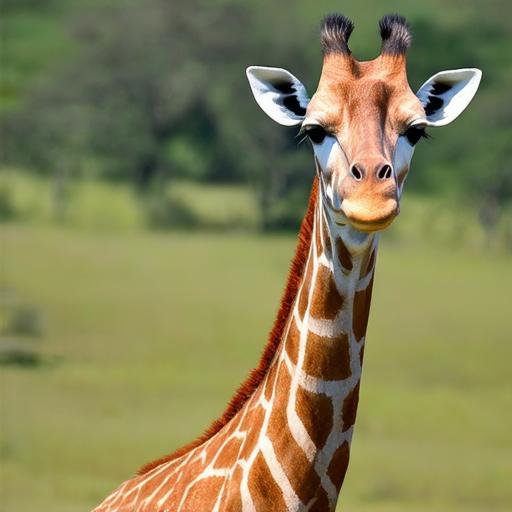} &
    \includegraphics[width=0.8in]{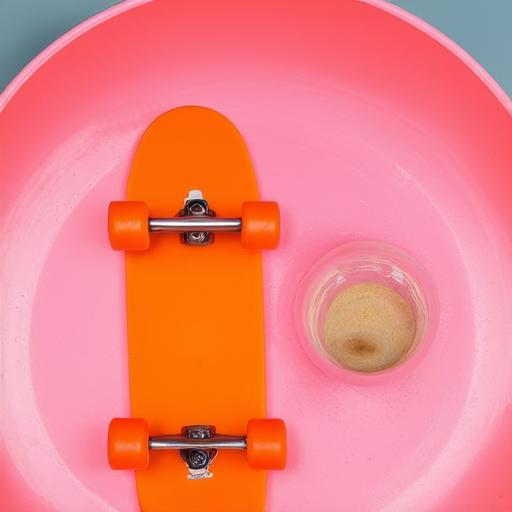} \\
    \rotatebox{90}{\hspace{0.15in}DiffusionNFT} &
    \includegraphics[width=0.8in]{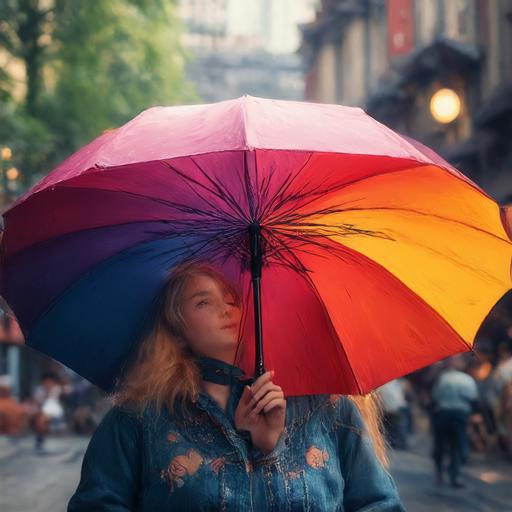} &
    \includegraphics[width=0.8in]{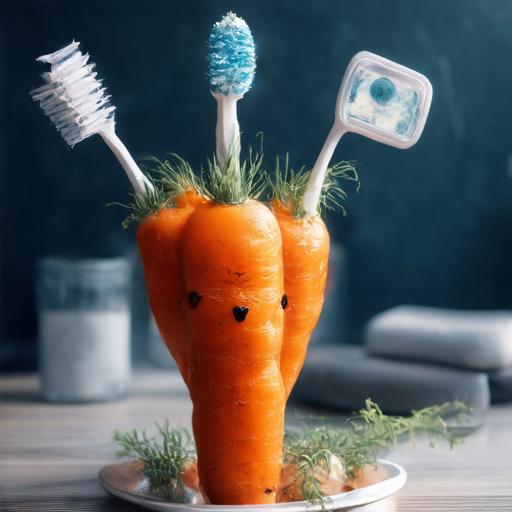} &
    \includegraphics[width=0.8in]{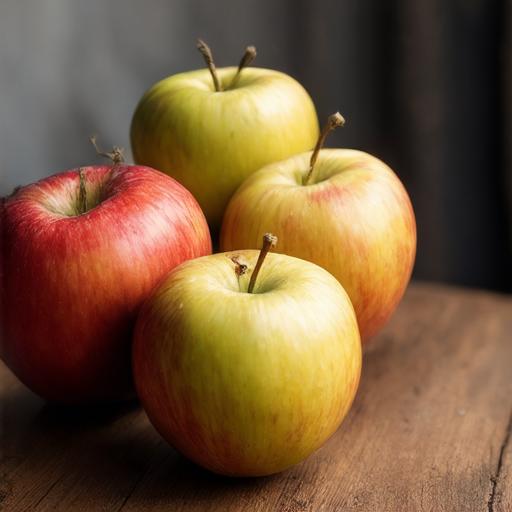} &
    \includegraphics[width=0.8in]{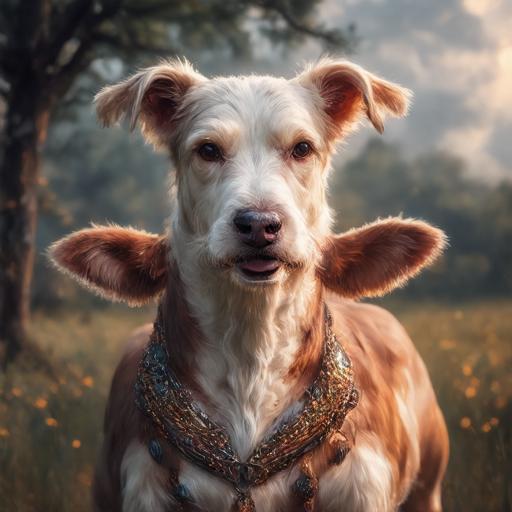} &
    \includegraphics[width=0.8in]{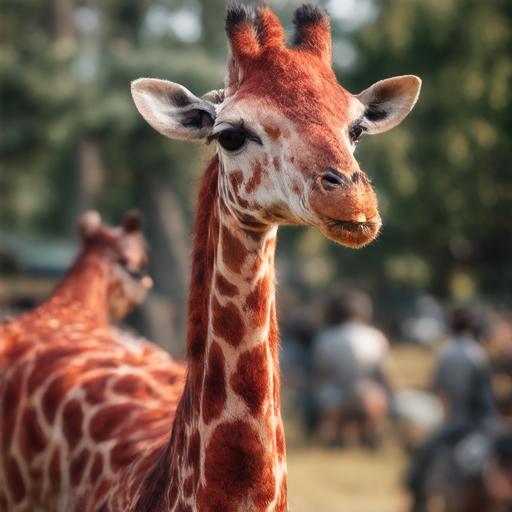} &
    \includegraphics[width=0.8in]{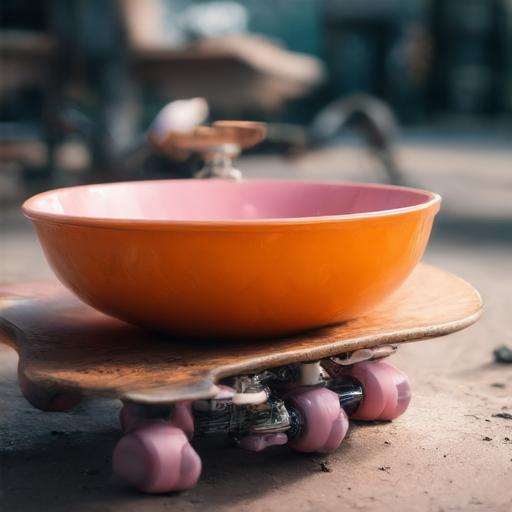} \\
    \rotatebox{90}{\color{red}AdvantageFlow (ours)} &
    \includegraphics[width=0.8in]{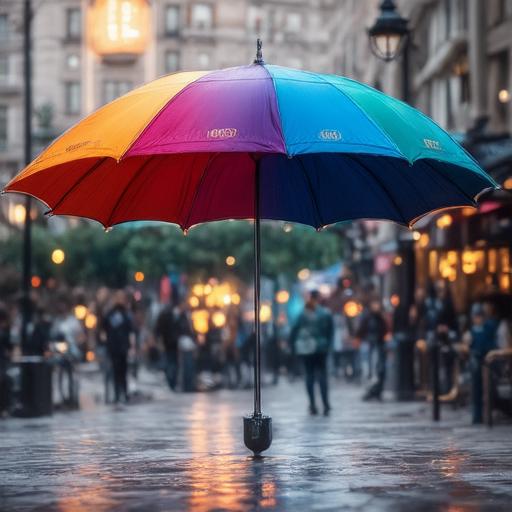} &
    \includegraphics[width=0.8in]{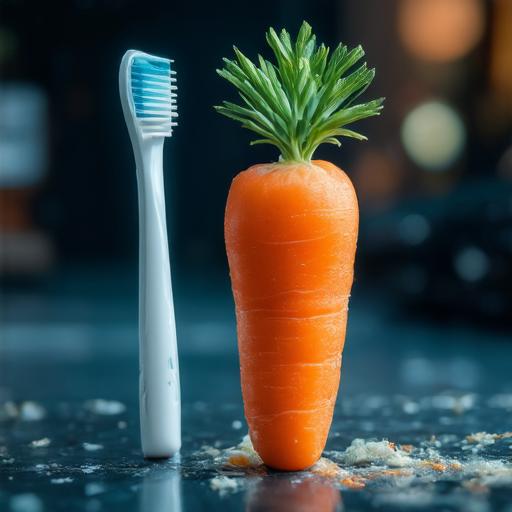} &
    \includegraphics[width=0.8in]{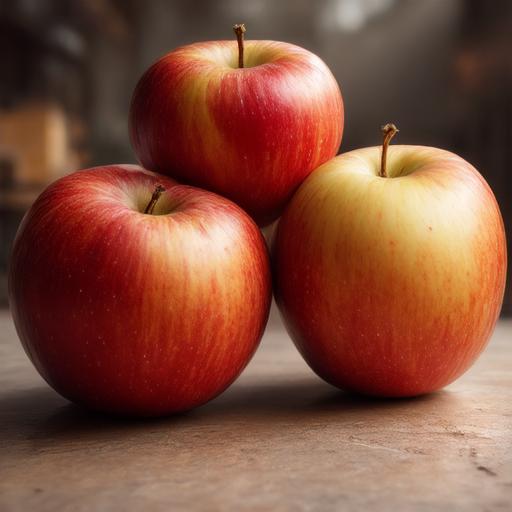} &
    \includegraphics[width=0.8in]{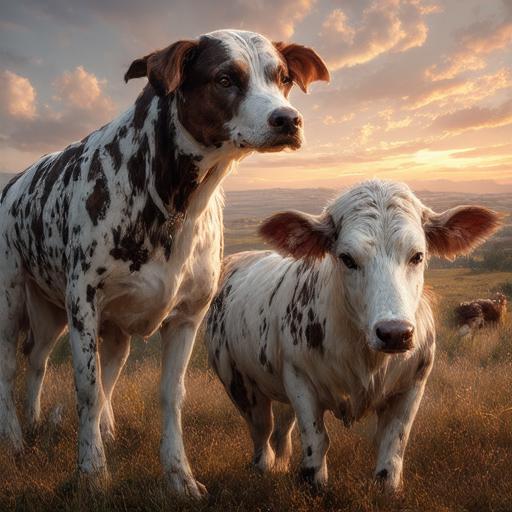} &
    \includegraphics[width=0.8in]{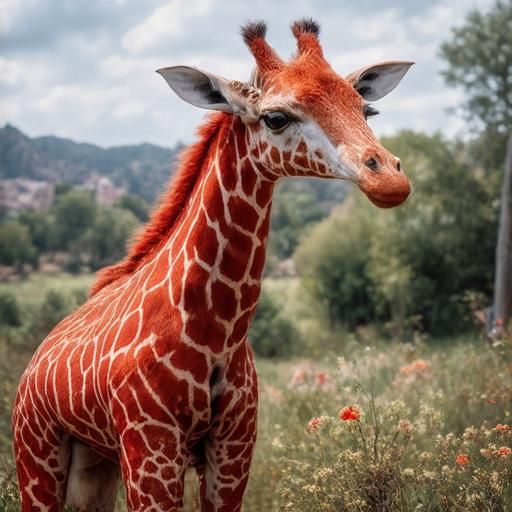} &
    \includegraphics[width=0.8in]{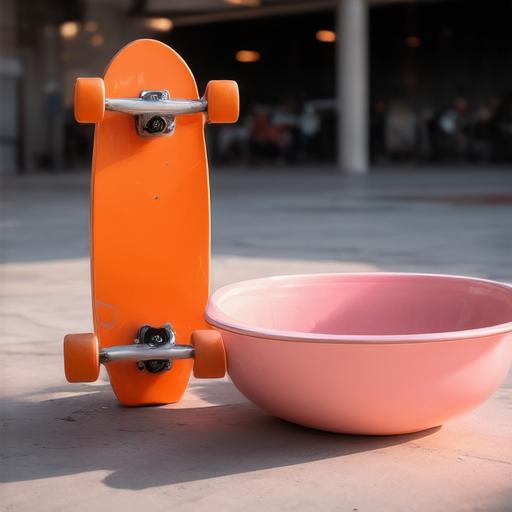}
  \end{tabular}} \vspace{0.1in}
  \caption{Images generated by \advantageflow, \diffusionnft, and base model (Stable Diffusion 3.5 Medium) with cfg. \advantageflow improves object generation, counting, color understanding, and position understanding; all without classifier free guidance (cfg). The images are obtained after training for $60$ A100 GPU hours.}
  \label{fig:teaser}
\end{figure}

\section{Introduction}
\label{sec:introduction}

Flow and diffusion models are standard models for generating high-quality images \citep{lipman2023flow,liu2023rectified,albergo2023stochastic}. The models are usually pre-trained by minimizing the flow-matching loss and then post-trained to improve human preference alignment, prompt following, text rendering, and image composition. At inference time, images are generated by sampling noise and following a learned denoising trajectory.

Given a pre-trained image generation model and a reward model, we want to adapt the pre-trained model to generate more rewarding images, while preserving its quality and diversity. Most \emph{reinforcement learning (RL)} algorithms for flow models optimize this sampling process by treating image generation as a sequential process along the reverse denoising trajectory and apply a variant of policy gradients \citep{black2024ddpo,fan2023dpok,flowgrpo,clark2024draft,prabhudesai2023alignprop,domingoenrich2024adjoint}. These methods are powerful but complex. For instance, they need to introduce additional stochasticity just to compute policy gradients. Moreover, since only the final generated image is rewarded, reward attribution to intermediate denoising steps where mistakes happen must be learned from data. Several recent papers try to improve reward attribution in RL for flow models \citep{savani2026stepwiseflowgrpo,he2025tempflowgrpo,zhou2025fine}. Forward-process RL offers a simpler alternative. Instead of optimizing along the reverse sampling process, we sample complete images, score them with rewards, and train the model using the same forward-process loss as in pre-training but with reward-dependent weights.

We introduce \advantageflow, a forward-process RL algorithm for rectified flow models. The key idea is to sample a batch of images per prompt and compute their advantages: how much better or worse each image is relative to the batch average. Then we train the model to generate images with higher advantages more frequently. The optimized objective is a \emph{regularized advantage-weighted prediction loss}. The advantage weighting can be viewed as fitting a local linearly tilted target distribution that shifts probability mass towards more rewarding samples. The regularization arises naturally as a variance reduction step: it replaces noisy per-sample targets with their conditional means in the reward-independent part of the loss. It also ensures that the loss is strictly convex per sample. We evaluate \advantageflow on text-to-image generation with Stable Diffusion 3.5 Medium and FLUX.1. Our algorithm improves the quality of generated images across all tested reward models. It outperforms both \flowgrpo \citep{flowgrpo}, and state-of-the-art forward-process RL baselines based on negative-aware fine-tuning \citep{diffusionnft} and advantage-weighted matching \citep{xue2025awm}. To mitigate a potential concern of reward hacking, we complement our evaluation by GenEval \citep{ghosh23geneval}.

Our contributions are: (1) We introduce \advantageflow, a least-squares algorithm for reinforcement learning in flow models using the forward process. (2) We derive the algorithm from a local reward-improving target distribution, where the rollout regularization arises as a variance reduction step for noisy sampled targets. (3) We evaluate \advantageflow on text-to-image generation with Stable Diffusion 3.5 Medium and FLUX.1. \advantageflow mostly outperforms all baselines and generates high-quality images.

\section{Background}
\label{sec:background}

We briefly review flow matching for image generation \citep{lipman2023flow,liu2023rectified,albergo2023stochastic}. \emph{Rectified flow matching models} learn continuous-time normalizing flows by training on a linear interpolation of clean images and Gaussian noise. Let $c$ be a prompt and $q(\cdot \mid c)$ be a distribution over clean images. Let $x_0 \sim q(\cdot \mid c)$ be a random clean image, $\epsilon \sim \cN(0, I)$ be Gaussian noise, and $t \sim [0, 1]$ be a random time. Then the flow matching objective is minimizing
\begin{align}
  \E_{x_0 \sim q(\cdot \mid c), \, \epsilon, \, t}
  [\norm{v_\theta(x_t, t, c) - (\epsilon - x_0)}^2]\,,
  \label{eq:flow-matching loss}
\end{align}
where $x_t = (1 - t) x_0 + t \epsilon$ is a random interpolation of $x_0$ and $\epsilon$, $v_\theta(x_t, t, c)$ is the \emph{predicted velocity} at $x_t$, $\theta$ is the \emph{learned model parameter}, and $\epsilon - x_0$ is the ground-truth velocity. The objective \eqref{eq:flow-matching loss} can also be viewed as prediction loss minimization. Specifically, the \emph{predicted clean image} at $x_t$ is $\ftheta(x_t, t, c) = x_t - t \, v_\theta(x_t, t, c)$ and therefore $\norm{\ftheta(x_t, t, c) - x_0}^2 = t^2 \norm{v_\theta(x_t, t, c) - (\epsilon - x_0)}^2$. In this work, our goal is to develop and justify algorithms for minimizing
\begin{align}
  \E_{x_0 \sim q(\cdot \mid c), \, \epsilon, \, t}
  [\normw{\ftheta(x_t, t, c) - x_0}{2}^2]
  \label{eq:prediction loss}
\end{align}
weighted by rewards. Without the rewards, the minimizer is the conditional mean of $x_0$.

Sampling in flow models is implemented using a \emph{deterministic ODE}. Specifically, it starts with noise $x_T$ and follows a denoising \emph{trajectory} $x_T, x_{T -\Delta t}, \dots, x_0$, where
\begin{align}
  x_{t - \Delta t}
  = x_t - v_\theta(x_t, t, c) \, \Delta t\,.
  \label{eq:ode}
\end{align}
The key idea in most existing RL algorithms for flow models, such as \flowgrpo, is treating this trajectory as a sequential process with \emph{reward} $r(x_0, c)$, for final generated image $x_0$ given prompt $c$. There are two challenges. The first challenge is that policy gradients \citep{williams92simple} cannot be applied to deterministic trajectories, such as the one generated by \eqref{eq:ode}. To address this, \flowgrpo replaces \eqref{eq:ode} with an SDE
\begin{align}
  x_{t - \Delta t}
  = x_t - \left[v_\theta(x_t, t, c) - \frac{\sigma_t^2}{2 t} \hat{x}_1\right]
  \Delta t + \sigma_t \sqrt{\Delta t} \xi\,, \quad
  \hat{x}_1
  = x_t + (1 - t) \, v_\theta(x_t, t, c)\,,
  \label{eq:sde}
\end{align} 
where $\xi\sim\mathcal N(0,I)$ and $\sigma_t$ is the noise schedule of the
stochastic sampler. This defines a Gaussian transition $\pi_\theta(x_{t - \Delta t} \mid x_t, c)$ at denoising step $t$ with covariance $\sigma_t^2 \Delta t \, I$, allowing policy gradients to be applied. The second challenge is that the reward $r(x_0, c)$ is assigned only to the final generated image, while $\pi_\theta(x_{t - \Delta t} \mid x_t, c)$ needs to be optimized. This creates a reward attribution problem and many recent works try to improve it \citep{savani2026stepwiseflowgrpo,he2025tempflowgrpo,zhou2025fine}.

Forward-process RL has neither of these problems. It uses the ODE sampler in \eqref{eq:ode} to generate clean images and thus is better aligned with the inference-time denoising. The clean image is predicted directly using $\ftheta(x_t, t, c)$ and therefore no reward attribution is needed.

\section{Algorithm}
\label{sec:algorithm}

We introduce \advantageflow, a forward-process RL algorithm for rectified flow models. The key idea is to minimize the prediction loss in \eqref{eq:prediction loss} weighted by advantages and regularized by the rollout policy, which convexifies the objective and makes its optimization stable. The \emph{rollout policy} is used for sampling training images. To prevent the loss of capabilities that we do not optimize for, we also regularize by a \emph{reference policy}. We denote predicted clean images under the learned, rollout, and reference policies by $\ftheta$, $\fold$, and $\fref$, respectively.

\subsection{Intuition}
\label{sec:intuition}

Suppose that we sample multiple images per prompt and score them with rewards. Some images are better than the average and some are worse. We want to generate the good images more often than the bad ones. A natural idea is to minimize the prediction loss in \eqref{eq:prediction loss} weighted by advantages, rewards minus the mean reward, which upweight above-average images and downweight the rest. Specifically, the loss in \eqref{eq:prediction loss} becomes $A(x_0, c) \norm{\ftheta(x_t, t, c) - x_0}^2$, where $A(x_0, c)$ is the advantage of clean image $x_0$ given prompt $c$. Perhaps surprisingly, this is not sufficient. When an image has a negative advantage, $A(x_0, c) \norm{\ftheta(x_t, t, c) - x_0}^2$ becomes negative. This can cause instability and divergence because the loss is not convex in $\ftheta$ anymore (\cref{sec:strict convexity}).

To get around this issue, we incorporate the advantage through fitting a target distribution that upweights high-advantage samples (\cref{sec:additive target}). Specifically, we linearly tilt the rollout distribution $\pold$: its density at $x_0$ is reweighted by $1 + \eta A(x_0,c)$ for $\eta > 0$. We cannot sample from the tilted distribution directly. Therefore, we sample from $\pold$ and reweight the samples using importance weights $1 + \eta A(x_0,c)$, as can be seen in \eqref{eq:importance sampling}. When the importance weight is expanded, the loss becomes $\eta A(x_0, c) \norm{\ftheta(x_t, t, c) - x_0}^2 + \norm{x_0 - \ftheta(x_t, t, c)}^2$, where the latter term can be rewritten in expectation as $\norm{\ftheta(x_t, t, c) - \fold(x_t, t, c)}^2$ and provably reduces variance (\cref{sec:variance reduction}). It can be also interpreted as rollout policy regularization. We also regularize with respect to the reference policy $\fref(x_t, t, c)$. This prevents the loss of capabilities not captured by rewards.

Putting it all together, our objective is
\begin{align}
  \E_{x_0 \sim \pold(\cdot \mid c), \, \epsilon, \, t}
  [\ell_{\theta, \, \theta_\old}(A(x_0, c), x_0, x_t, t, c)]\,,
  \label{eq:advantageflow loss}
\end{align}
where $\ell_{\theta, \, \theta_\old}(A(x_0, c), x_0, x_t, t, c) =$
\begin{align*}
  A(x_0, c) \norm{\ftheta(x_t, t, c) - x_0}^2 + 
  \gamma \norm{\ftheta(x_t, t, c) - \fold(x_t, t, c)}^2 +
  \lambda \norm{\ftheta(x_t, t, c) - \fref(x_t, t, c)}^2\,.
  \nonumber
\end{align*}
We refer to the three terms as advantage-weighted prediction loss, rollout regularization, and reference regularization. The \emph{advantage-weighted prediction loss} fits the policy to generated images weighted by their advantages. Positive $A(x_0, c)$ pull the policy towards good images and negative $A(x_0, c)$ push it away from the bad images. The \emph{rollout regularization} penalizes deviations from the rollout policy. It ensures that the loss is strictly convex per sample, by making the total quadratic coefficient $A(x_0, c) + \gamma + \lambda$ positive (\cref{sec:strict convexity}). The \emph{reference regularization} penalizes deviations from the initial pre-trained model and thus prevents forgetting behaviors not captured by rewards. The weights $\gamma, \lambda \geq 0$ control the amount of rollout and reference regularization, respectively.

\subsection{AdvantageFlow}
\label{sec:advantageflow}

The pseudo-code of our algorithm is given in \cref{alg:advantageflow} and we call it \advantageflow, since it optimizes the flow model using advantages. The algorithm is iterative. In the outer loop, we sample $L$ prompts, generate $K$ images per prompt using the rollout policy, score the images by rewards, and compute their advantages. In the inner loop, we take $M \geq 1$ gradient steps on the learned policy using the same batch of images, before updating the rollout policy. Since the objective \eqref{eq:advantageflow loss} is a fixed expectation under the rollout policy $\pold$, multiple inner steps require no importance correction: our approach is off-policy by design.

\textbf{Rollout policy.} The images are generated by the rollout policy, which is parameterized by $\theta_\old$. The parameter $\theta_\old$ is slowly updated using an \emph{exponential moving average (EMA)} of $\theta$, as in \citet{diffusionnft}. Both $\theta$ and $\theta_\old$ parametrize velocity models $\vtheta$ and $\vold$. Therefore, they can be used to sample images, through \eqref{eq:ode}, or predict clean images, through $\ftheta$ and $\fold$.

\textbf{Advantages.} The advantages can be computed in many different ways \citep{shao2024deepseekmath,liu2025understanding,flowgrpo}. As in \citet{diffusionnft}, we standardize them per-batch as follows. Let $r^{i, k}$ denote the reward for image $k \in [K]$ and prompt $c_i$, where $i \in [L]$. Then the corresponding advantage is
\begin{align}
  \!\!\!\!\!
  A^{i, k}
  = \clip\left(\frac{r^{i, k} - \hat{r}^i}{Z}, \, -1, \, 1\right)\,, \quad
  \hat{r}^i
  = \frac{1}{K} \sum_{k = 1}^K r^{i, k}\,, \quad
  Z
  = \sqrt{\frac{1}{K L} \sum_{i = 1}^L \sum_{k = 1}^K (r^{i, k} - \hat{r}^i)^2}\,,
  \label{eq:empirical advantage}
\end{align}
where $\hat{r}^i$ is the per prompt mean and $Z$ is the batch standard deviation. The advantage of a sample is positive if its reward is above the average and negative if it is below.

\textbf{Prediction loss minimization.} We minimize a finite-sample approximation of \eqref{eq:advantageflow loss}, where the expectation over prompts, clean images, time, and Gaussian noise is approximated over $K L$ samples in a batch. The images $x_0^{i, k}$ and advantages $A^{i, k}$ are computed once per outer iteration; the time $t$ and noise $\epsilon$ are drawn fresh per sample and per inner-loop step.

\textbf{Computational cost.} The time complexity of each outer iteration of \advantageflow{} is linear in $K L$ for sampling the images, and linear in $M K L$ for the inner loop, which requires predictions by $\ftheta$, $\fold$, and $\fref$ for each of the $K L$ noisy images at each of the $M$ inner steps.

\textbf{Space complexity.} The space complexity is mostly driven by the dimensionality of $\theta$ and $\theta_\old$. In our work, both are encoded using LoRA \citep{hu22lora} on the top of the initial pre-trained model, which represents the reference policy.

\begin{algorithm}[t]
  \caption{\advantageflow}
  \label{alg:advantageflow}
  \begin{algorithmic}[1]
    \State \textbf{Input:} Pre-trained model $\theta_{\mathrm{ref}}$, reward model $r$, prompt dataset $\cC$, batch size $L$, group size $K$, number of inner steps $M \geq 1$, regularization strengths $\gamma, \lambda \geq 0$, and EMA update rate $\rho \geq 0$
    \State Initialize $\theta \gets \theta_{\mathrm{ref}}$ and $\theta_\old \gets \theta_{\mathrm{ref}}$
    \For{each outer iteration}
      \State Sample $L$ prompts $c_1, \dots, c_L \sim \cC$
      \State For each $c_i$, generate $K$ images $x_0^{i, 1}, \dots, x_0^{i, K}$ using \eqref{eq:ode} with rollout velocity $\vold$
      \State Compute per-image advantage $A^{i, k}$ using \eqref{eq:empirical advantage}, for all $(i, k) \in [L] \times [K]$
      \For{$M$ inner steps}
        \ForAll{$(i, k) \in [L] \times [K]$}
          \State $x_t^{i, k} \gets (1 - t^{i, k}) \, x_0^{i, k} + t^{i, k} \epsilon^{i, k}$, for freshly sampled $t^{i, k}$ and $\epsilon^{i, k}$
        \EndFor
        \State Update $\theta$ by gradient descent on
        $\frac{1}{K L} \sum_{i = 1}^L \sum_{k = 1}^K
        \ell_{\theta, \, \theta_\old}(A^{i, k}, x_0^{i, k}, x_t^{i, k}, t^{i, k}, c_i)$
      \EndFor
      \State Update rollout policy as
      $\theta_\old \gets \rho \, \theta_\old + (1 - \rho) \, \theta$
    \EndFor
  \State \textbf{Output:} $\theta$
  \end{algorithmic}
\end{algorithm}

\subsection{Strict Convexity}
\label{sec:strict convexity}

The loss inside \eqref{eq:advantageflow loss} is quadratic in $\ftheta$ for any sample, with a \emph{quadratic coefficient} $A(x_0, c) + \gamma + \lambda$. It is strictly convex, and therefore has a unique minimum in the prediction space, when
\begin{align}
  A(x_0, c) + \gamma + \lambda
  > 0\,.
  \label{eq:well-conditioned}
\end{align}
Without rollout regularization, $\gamma = 0$ and negative advantages $A(x_0, c)$ could make the quadratic coefficient negative, thus turning the loss concave and making its optimization unstable.

Since the advantages in \eqref{eq:empirical advantage} are clipped to $[-1, 1]$, any $\gamma > 1$ would guarantee that \eqref{eq:well-conditioned} holds for any sample, irrespective of $\lambda \geq 0$. In our experiments, we set $\gamma$ \emph{adaptively per prompt}. Specifically, it is $\gamma(c) = A_{\max}(c) + \gamma_0$, where $A_{\max}(c)$ is the maximum absolute value of advantages per prompt $c$ and $\gamma_0$ is a small constant that guarantees a minimum level of strict convexity. We show that this choice outperforms a fixed $\gamma$ in \cref{sec:experiments}.

\section{Analysis}
\label{sec:analysis}

This section provides theory supporting our design choices in \cref{sec:intuition}: linear tilting provides a natural-gradient justification for advantage weighting and using the rollout prediction as the regression target reduces variance. Neither result is needed to follow the algorithm.

\subsection{Linearly Tilted Target Distribution}
\label{sec:additive target}

Fix a prompt $c$ and let $\pold(\cdot \mid c)$ be the rollout distribution. We define the centered advantage
\begin{equation}
  A(x_0, c)
  = r(x_0, c) - \E_{x'_0 \sim \pold(\cdot \mid c)}[r(x'_0, c)]\,.
  \label{eq:conditional advantage}
\end{equation} 
Consider the distribution obtained by tilting $\pold$ by a factor linear in the advantage,
\begin{equation}
  \qeta(x_0 \mid c)
  = (1 + \eta A(x_0, c)) \, \pold(x_0 \mid c)\,.
  \label{eq:additive target}
\end{equation}
It increases the probability of above-average samples ($A(x_0,c)>0$) and decreases that of below-average samples ($A(x_0,c)<0$), proportionally to how far each sample is from the mean reward. Since the advantage is centered, $\E_{x_0 \sim \pold(\cdot \mid c)}[A(x_0,c)]=0$, the perturbation integrates to zero and $\qeta$ is automatically normalized: $\int \qeta(x_0 \mid c)\,\dif x_0 = \int \pold\,\dif x_0 + \eta\int A(x_0,c)\,\pold\,\dif x_0 = 1 + 0 = 1$.  If $1 + \eta A(x_0, c) \geq 0$ for all $x_0$ (equivalently, the rescaled form of the pointwise positivity condition $A(x_0,c) + \gamma > 0$ from \cref{sec:strict convexity} with $\gamma = 1/\eta$ and $\lambda = 0$), $\qeta(x_0 \mid c) \geq 0$ and therefore a valid probability distribution. It also improves the expected reward by $\eta\,\Var_{\pold(\cdot\mid c)}(r(\cdot,c))\geq 0$, since $\E_{\qeta}[r]-\E_{\pold}[r] = \eta\,\E_{\pold}[A(x_0,c)\,r] = \eta\,\E_{\pold}[A(x_0,c)^2] = \eta\,\Var_{\pold}(r)$, using $\E_{\pold}[A(x_0,c)]=0$ and the definition of variance. The linear tilt $\qeta$ in \eqref{eq:additive target} is the first-order natural-gradient update for maximizing expected reward under the Fisher-Rao metric on probability distributions \citep{amari1998natural, kakade2001natural}. We make this connection precise.

\begin{proposition}[Fisher-Rao natural-gradient direction]
\label{prop:fisher-rao}
Fix a prompt $c$ and let $F(p) = \E_p[r(\cdot, c)]$. Under the Fisher-Rao metric, the natural-gradient direction of $F$ at $p = \pold(\cdot \mid c)$ is the tangent vector
\begin{align*}
  \delta p(x_0)
  = A(x_0, c) \, \pold(x_0 \mid c)\,, \qquad
  A(x_0, c) = r(x_0, c) - \E_{\pold(\cdot \mid c)}[r(\cdot, c)]\,.
\end{align*}
A first-order step of size $\eta$ along this direction yields the additive tilt $\qeta(\cdot \mid c) = (1 + \eta A(\cdot, c)) \, \pold(\cdot \mid c)$ from \eqref{eq:additive target}.
\end{proposition}

The proof is in \cref{app:fisher-rao}. The Fisher-Rao metric is the canonical Riemannian metric on probability distributions and underlies natural policy gradient methods in reinforcement learning \citep{kakade2001natural}. \cref{prop:fisher-rao} carries this picture over to flow models: \advantageflow can be read as natural-gradient ascent on expected reward, with the rollout distribution as the base point and $\eta$ as the step size.

Formally, the prediction-loss objective under the tilted distribution is
\begin{align}
  \cL_{\qeta}(\theta)
  = \E_{x_0 \sim \qeta(\cdot \mid c), \, \epsilon, \, t}
  [\norm{x_0 - \ftheta(x_t, t, c)}^2]\,.
  \label{eq:qeta-loss}
\end{align}
Changing measure to the rollout distribution gives 
\begin{align}
  \cL_{\qeta}(\theta)
  = \E_{x_0 \sim \pold, \, \epsilon, \, t}
  [(1 + \eta A(x_0,c)) \, \norm{x_0 - \ftheta(x_t, t, c)}^2]\,.
  \label{eq:importance sampling}
\end{align}
Suppressing the dependence of $A$ on $x_0,c$ and expanding the weight gives
\begin{align}
  \cL_{\qeta}(\theta)
  = \underbrace{\E_{x_0 \sim \pold, \, \epsilon, \, t}
  [\norm{x_0 - \ftheta(x_t, t, c)}^2]}_{\text{reward-independent}} {} + {}
  \eta \, \underbrace{\E_{x_0 \sim \pold, \, \epsilon, \, t}
  [A(x_0,c) \, \norm{x_0 - \ftheta(x_t, t, c)}^2]}_{\text{reward-dependent}}\,.
  \label{eq:decomposed loss}
\end{align}
\subsection{Variance Reduction by Rollout Policy}
\label{sec:variance reduction}

Variance reduction in policy optimization is a well-studied topic \citep{sutton00policy,baxter01infinitehorizon,munos06geometric}. The reward-independent term is defined on individual samples $x_0$, but the population minimizer of this term is the conditional mean $\E_{x_0 \sim \pold, \epsilon}[x_0 \mid x_t, t, c]$. If the rollout policy is well-trained, its prediction $\fold(x_t, t, c)$ approximates this conditional mean. We then are better off by replacing $x_0$ with the prediction $\fold(x_t, t, c)$ in reward-independent term of \cref{eq:decomposed loss}.

\begin{proposition}[Variance reduction]
\label{prop:two-anchor}
Fix a prompt $c$ and suppose $\fold(x_t, t, c) = \E_{x_0 \sim \pold(\cdot \mid c), \epsilon}[x_0 \mid x_t, t, c]$. Up to a $\theta$-independent constant,
\begin{align*}
  \E_{x_0 \sim \pold, \, \epsilon, \, t}[\norm{x_0 - \ftheta(x_t, t, c)}^2]
  = \E_{x_0 \sim \pold, \, \epsilon, \, t}[\norm{\fold(x_t, t, c) - \ftheta(x_t, t, c)}^2]\,.
\end{align*}
Moreover, the single-sample gradient of the right-hand side has no larger variance than that of the left-hand side.
\end{proposition}

The proof is in \cref{app:two-anchor-equivalence}. The key observation is that $x_0$ is a single noisy draw from $\pold(\cdot \mid c)$, while $\fold(x_t, t, c)$ is trained to match its conditional mean. Using $\fold$ as the regression target removes sample noise. The population minimizer is unchanged, and the gradient estimator is lower-variance by the Rao-Blackwell argument \citep{casella96raoblackwellisation}. Substituting this into \eqref{eq:decomposed loss} and rescaling by $\eta = 1 / \gamma$ gives the objective of \advantageflow in \cref{sec:algorithm} without reference regularization,
\begin{align}
  \E_{x_0 \sim \pold, \, \epsilon, \, t}
  [A(x_0,c) \, \norm{x_0 - \ftheta(x_t, t, c)}^2 +
  \gamma \, \norm{\ftheta(x_t, t, c) - \fold(x_t, t, c)}^2]\,.
  \label{eq:afm-population-loss}
\end{align}

\section{Experiments}
\label{sec:experiments}

Most of our experiments are with Stable Diffusion 3.5 Medium \citep{esser24scaling} and $512 \times 512$ images, in the \diffusionnft code base\footnote{\url{https://github.com/NVlabs/DiffusionNFT}} \citep{diffusionnft}. We experiment with two kinds of reward models: rule- and model-based. \emph{OCR} is ruled-based and evaluates text rendering. \emph{GenEval} \citep{ghosh23geneval} is ruled-based and evaluates image composition. \emph{PickScore} \citep{kirstain23pickapic}, \emph{HPSv2.1} \citep{wu23human}, and \emph{CLIPScore} \citep{hessel21humancentric} are learned and evaluate image quality, human preference, and text-to-image alignment. For OCR, we use their training and test sets. For other reward models, we use the Pick-a-Pic dataset \citep{kirstain23pickapic}.

\advantageflow is implemented as described in \cref{sec:advantageflow,sec:strict convexity}. We compare it to \flowgrpo \citep{flowgrpo}, arguably the most popular RL algorithm for flow models; and state-of-the-art forward-process RL algorithms \diffusionnft \citep{diffusionnft} and \awm \citep{xue2025awm}. All algorithms are trained on $10$ denoising steps (SDE in \flowgrpo and ODE otherwise) and evaluated on $40$. The learned and rollout policies are implemented using LoRA with $\alpha = 64$ and $r = 32$. We depart from the setting of \citet{diffusionnft} only slightly by having an effective batch size $512$: $32$ prompts ($1.5$ times fewer) with $16$ samples per prompt ($1.5$ times fewer). This leads to almost identical results at less than a half of their computational cost. We run all algorithms for $250$ iterations. Each run takes up to $70$ A100 GPU hours and all runs can be reproduced in a week on a single $8$-GPU A100 node. To reduce the variability of runs, we set the reference regularization strength in all forward-process RL algorithms to $\lambda = 0.001$. No other hyper-parameters have been modified. All reported rewards are measured on test sets and plotted as a function of training time.

\begin{figure}[t!]
  \centering
  \includegraphics[width=5.6in]{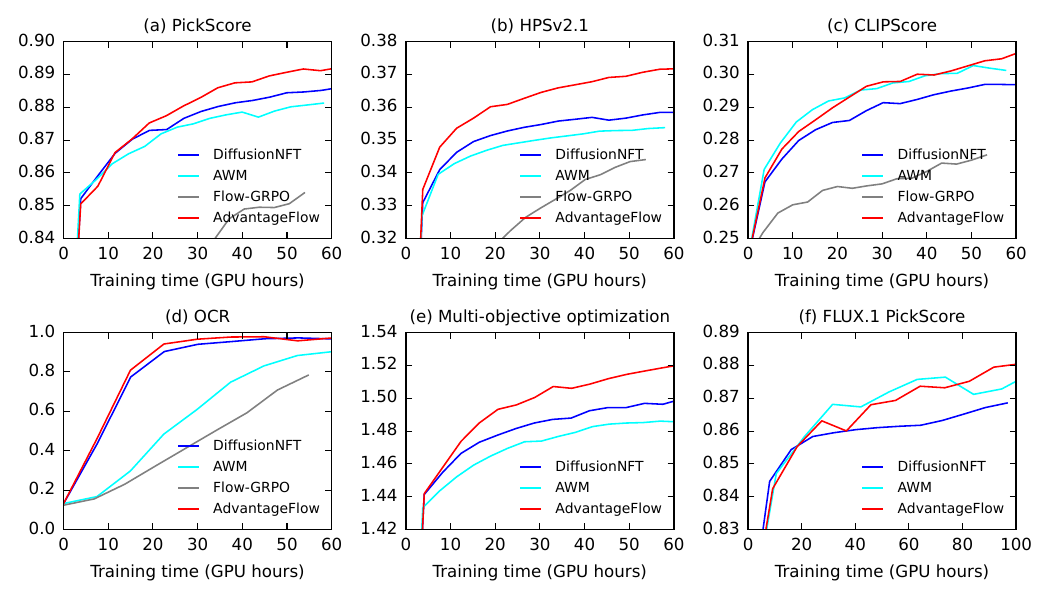} \\ \vspace{-0.1in}
  \caption{Evaluation on PickScore, HPSv2.1, CLIPScore, OCR, and multi-objective problems.}
  \label{fig:main}
\end{figure}

\textbf{Baseline comparison.} We start by comparing \advantageflow to all baselines. Our results across all reward models are reported in \cref{fig:main}a-d.\footnote{The OCR experiment in \citet{diffusionnft} is with \diffusionnft for $\beta = 0.1$. To get a similar behavior in \advantageflow (\cref{sec:related work}), we multiply its rollout regularization constant by $0.1$.} We observe two main trends. First, all forward-process RL algorithms outperform \flowgrpo by a large margin, which is consistent with observations in prior works \citep{diffusionnft,xue2025awm}. Second, \advantageflow outperforms \diffusionnft in all problems and \awm in $3$ problems out of $4$. \advantageflow and \awm are comparable in the CLIPScore problem (\cref{fig:main}c). Some gains are major. For instance, in the HPSv2.1 problem (\cref{fig:main}b), \advantageflow attains the highest reward of \diffusionnft in a quarter of its computational cost. Because the methods differ only in their loss functions, this shows the superiority of our least-squares RL objective. In all remaining experiments, we only compare to the best performing baselines \diffusionnft and \awm.

\begin{figure}[t!]
  \centering
  \includegraphics[width=5.6in]{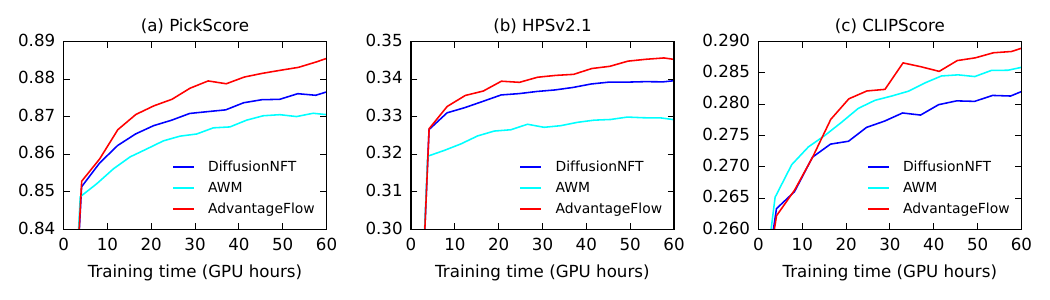} \\ \vspace{-0.1in}
  \caption{Evaluation on all rewards in multi-objective optimization where we optimize the sum of PickScore, HPSv2.1, and CLIPScore.}
  \label{fig:moo}
\end{figure}

\textbf{Multi-objective optimization.} Now we experiment with optimizing multiple rewards by training on the sum of PickScore, HPSv2.1, and CLIPScore. We report the combined reward in \cref{fig:main}e and individual rewards in \cref{fig:moo}. We observe two main trends. First, although we optimize the combined reward, all rewards in \cref{fig:moo} improve. This shows that simple variants of our method can be used for multi-objective optimization. Second, \advantageflow attains the highest reward of the best baseline \diffusionnft in a half of its computational cost.

\textbf{FLUX.1 experiments.} To show that our gains are not limited to Stable Diffusion 3.5 Medium, we experiment with PickScore maximization in FLUX.1 \citep{blackforestlabs25flux.1} in \cref{fig:main}f. We observe that \advantageflow outperforms \diffusionnft. In addition, it is comparable to \awm initially and outperforms it later. These experiments are done in Flow-Factory \citep{ping26flowfactory} with default implementations of all baselines.

\textbf{Rollout regularization strength.} One technical contribution of our work is rollout regularization and setting it adaptively (\cref{sec:strict convexity}). We ablate this in \cref{fig:rollout regularization} where \aflow{\gamma} corresponds to a fixed rollout regularization $\gamma$. When the regularization is low, \aflow{0.2} performs well in some problems (\cref{fig:rollout regularization}a) but is unstable in others (\cref{fig:rollout regularization}b). When the regularization is high, \aflow{5} is stable but conservative in both problems. In both problems, the adaptive regularization performs the best.

\begin{figure}[t!]
  \centering
  \includegraphics[width=5.6in]{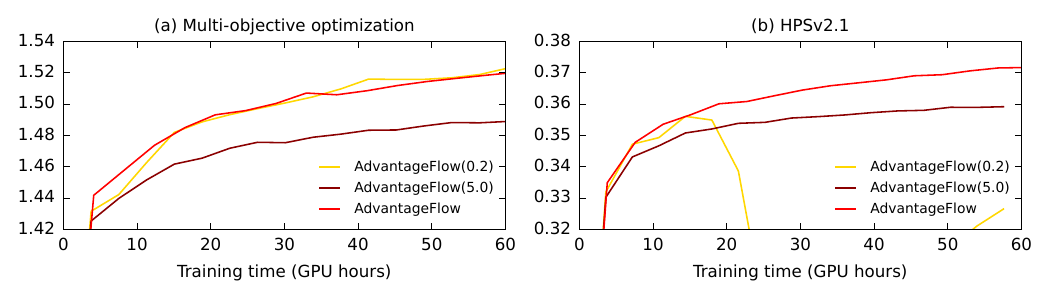} \\ \vspace{-0.1in}
  \caption{Comparison of \advantageflow with adaptive and fixed regularization on multi-objective optimization and HPSv2.1 problems.}
  \label{fig:rollout regularization}
\end{figure}

\textbf{Qualitative evaluation and GenEval.} To mitigate a potential concern of reward hacking, we show qualitative examples of image improvements in \cref{fig:teaser} and \cref{sec:additional experiments}, for the learned policies in \cref{fig:main}e. We also evaluate the policies on GenEval \citep{ghosh23geneval}. Our results are reported in \cref{tab:geneval} and we observe the following trends. First, \advantageflow improves over the original base model, with or without \emph{classifier-free guidance (cfg)}. Second, \advantageflow without cfg mostly outperforms the base model with cfg. This shows that RL can learn high-quality policies. Finally, we note that \advantageflow generally performs better than the baselines without cfg and comparably to \awm with cfg.

\begin{table}[t!]
  \centering
  {\footnotesize
  \begin{tabular}{l|rrrrrr|r} \hline
    Method & Single object & Two objects & Counting &
    Color & Position & Color binding & Overall \\ \hline
    \multicolumn{8}{c}{Without cfg} \\ \hline
    Base &
    0.750 & 0.191 & 0.181 & 0.478 & 0.055 & 0.072 & 0.273 \\
    \diffusionnft &
    \textbf{0.981} & 0.810 & 0.581 & 0.\textbf{824} & 0.197 & 0.557 & 0.647 \\
    \awm &
    0.953 & 0.792 & 0.571 & \textbf{0.824} & 0.192 & 0.555 & 0.637 \\
    \advantageflow &
    0.978 & \textbf{0.830} & \textbf{0.621} & 0.797 & 0.\textbf{257} & \textbf{0.615} & \textbf{0.673} \\ \hline
    \multicolumn{8}{c}{With cfg} \\ \hline
    Base &
    0.981 & 0.772 & 0.587 & 0.805 & 0.190 & 0.525 & 0.631 \\
    \diffusionnft &
    \textbf{0.987} & 0.911 & 0.687 & 0.843 & 0.247 & 0.645 & 0.710 \\
    \awm &
    \textbf{0.987} & 0.919 & 0.737 & 0.861 & \textbf{0.292} & \textbf{0.685} & \textbf{0.737} \\
    \advantageflow &
    0.984 & \textbf{0.929} & \textbf{0.743} & \textbf{0.864} & \textbf{0.292} & 0.655 & 0.734 \\ \hline
  \end{tabular}} \vspace{0.1in}
  \caption{GenEval evaluation of \advantageflow, \diffusionnft, \awm, and base model.}
  \label{tab:geneval}
\end{table}

\textbf{Random seed.} Due to the high computational cost of our experiments, all our experiments are with a single random seed $42$. To show sensitivity to multiple random seeds, we repeat the experiment in \cref{fig:main}e with other seeds in \cref{sec:additional experiments}.

\section{Related Work}
\label{sec:related work}

Reward fine-tuning of diffusion and flow models has developed along three lines: trajectory-level RL on the reverse sampling process, reward backpropagation, and forward-process RL. Our method belongs to the third family. We briefly describe each and then discuss closest related methods.

\textbf{Trajectory-level RL.} DDPO \citep{black2024ddpo} and DPOK \citep{fan2023dpok} treat denoising as a finite-horizon Markov decision process and apply PPO-style updates to per-step Gaussian transitions. \flowgrpo{} \citep{flowgrpo} brings this idea to flow models by converting the sampling ODE into an SDE with tractable per-step likelihoods; later variants improve credit assignment and rollout efficiency \citep{xue2025dancegrpo,he2025tempflowgrpo,li2025mixgrpo,branchgrpo2025}. These methods optimize along the reverse trajectory, requiring credit assignment across many denoising steps. Forward-process RL avoids this by using the same prediction loss as pretraining.

\textbf{Reward backpropagation.} DRaFT \citep{clark2024draft}, AlignProp \citep{prabhudesai2023alignprop}, and Adjoint Matching \citep{domingoenrich2024adjoint} differentiate the reward through the sampler or an equivalent control objective. This is effective for differentiable rewards but does not directly apply to the black-box rule-based rewards, such as OCR, used in many text-to-image benchmarks.

\textbf{Forward-process RL.} The core idea (generate samples, score them, and retrain with the pretraining loss using reward-dependent weights) has appeared in several forms. Reward-weighted likelihood maximization \citep{lee2023aligning,peters2007rewardweighted} and ORW-CFM-W2 \citep{fan2025orwcfm} weight the flow-matching loss by reward values. Centered reward distillation \citep{zhu2026crd} and AC-Flow \citep{acflow2025} also fit in this view.

\diffusionnft \citep{diffusionnft} fixes a rollout policy $\vold$ and constructs implicit positive and negative velocities from a reward-derived optimality weight $r \in [0,1]$, $v_\theta^+ = (1-\beta)\vold + \beta v_\theta$ and $v_\theta^- = (1+\beta)\vold - \beta v_\theta$, fit to the flow-matching target $v$ by the branch loss $r\norm{v_\theta^+ - v}^2 + (1-r)\norm{v_\theta^- - v}^2$. This direct regression construction is motivated by classifier-free guidance: the single trainable velocity $v_\theta$ is pulled toward the flow-matching target through the positive branch and pushed away from it through the negative branch. The algorithm can take multiple gradient steps on a batch of samples collected under a fixed $\vold$ before that policy is updated by an exponential moving average. \advantageflow identifies an explicit reward-improved target distribution and fits the flow model to it directly (\cref{sec:algorithm}). In fact, after converting \diffusionnft to prediction space, its loss can be viewed as in \advantageflow with an advantage-dependent rollout-regularization coefficient (\cref{app:diffusionnft-algebra}), under additional assumptions.

\textbf{Advantage weighted matching} (\awm) \citep{xue2025awm} starts from the GRPO policy-gradient objective, treating the generated image $x_0$ as the action of a sequence-level policy $p_\theta(x_0 \mid c)$, and substitutes the flow-matching ELBO surrogate for its intractable log-likelihood. On-policy, this yields the advantage-weighted flow-matching gradient. Because the loss is derived as a policy gradient, reusing off-policy samples requires a correction, which AWM approximates by an importance ratio built from exponentiated differences of flow-matching losses. \advantageflow instead writes its loss directly as an expectation under the fixed rollout policy $\pold$, using the tilted target distribution in \eqref{eq:additive target}. So it is a stationary batch-optimization problem by construction and needs no importance correction for repeated optimization steps.

\section{Conclusions}
\label{sec:conclusions}

We introduce \advantageflow, a forward-process RL algorithm for rectified flow models based on advantage-weighted least squares. The algorithm fits a local additive reward-improving distribution, where the rollout regularization reduces variance and improves stability, by keeping the per-sample quadratic loss strictly convex. We evaluate \advantageflow on text-to-image generation, and show that it outperforms both \flowgrpo and state-of-the-art forward-process RL algorithms.

\textbf{Limitations and future work.} Our analysis is local and at a population level. We assume a centered population advantage and that the rollout predictor $\fold$ is equal to the conditional mean under $p_{\old}$. In practice, finite prompt groups, empirical reward normalization, clipped advantages, LoRA, and finite optimization are used. Therefore, while our theory motivates our objective, it is not the exact description of the finite-sample training dynamics.

The performance of \advantageflow, similarly to any other RL algorithm, depends on the quality of reward models, which we do not investigate. Learned reward models may be misspecified. On the other hand, rule-based reward models capture only a narrow aspect of image quality. Finally, the rollout and reference regularization strengths are free hyper-parameters of \advantageflow, and a more complete theory for selecting them is needed.

\newpage



\bibliographystyle{iclr2027_conference}
\bibliography{refs,brano}

\clearpage
\onecolumn
\appendix

\section{Additional Experiments}
\label{sec:additional experiments}

We show qualitative examples of image improvements for the learned policies in \cref{fig:main}e next. \cref{fig:teaser generation} shows improvements in object generation, \cref{fig:teaser composition} shows improvements in counting and position understanding, and \cref{fig:teaser color} shows improvements in color understanding.

\begin{figure}[t!]
  \centering
  {\tiny
  \begin{tabular}{@{\ }rc@{\ }c@{\ }c@{\ }c@{\ }c@{\ }c@{\ }}
    & hot dog &
    spoon &
    umbrella &
    stop sign and &
    chair and &
    toothbrush and \\
    & &
    &
    &
    fork &
    bench &
    carrot \\
    \rotatebox{90}{\hspace{0.1in}Base model + cfg} &
    \includegraphics[width=0.8in]{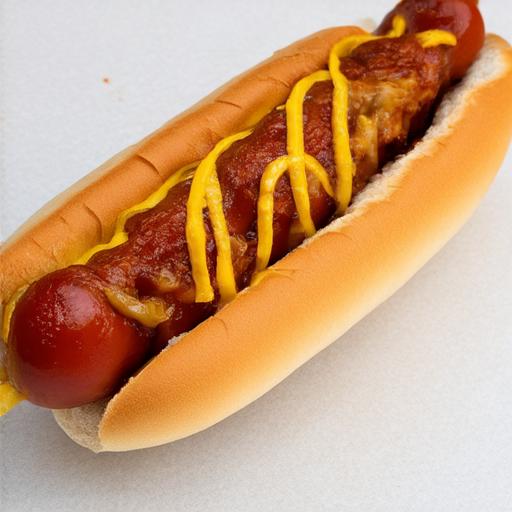} &
    \includegraphics[width=0.8in]{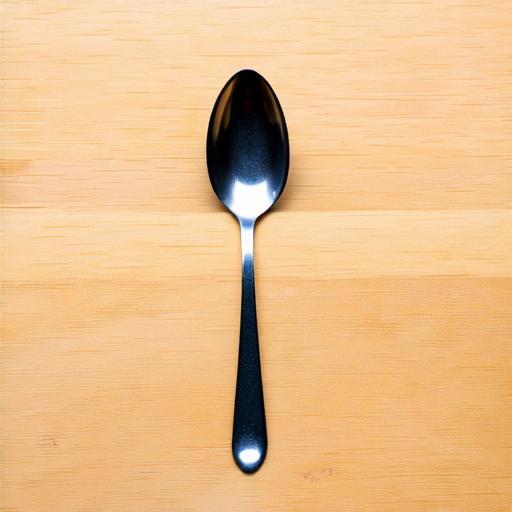} &
    \includegraphics[width=0.8in]{images/cfg_base/01572.jpg} &
    \includegraphics[width=0.8in]{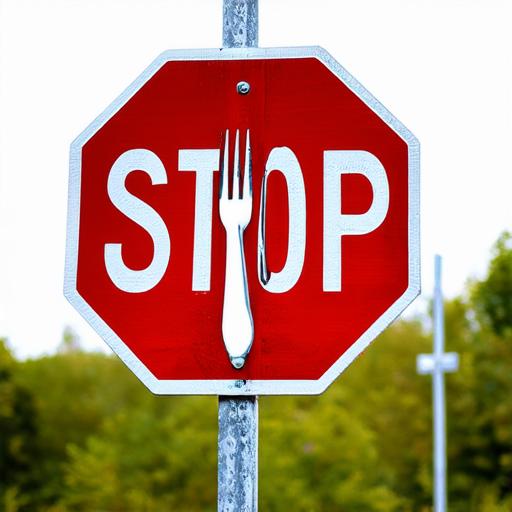} &
    \includegraphics[width=0.8in]{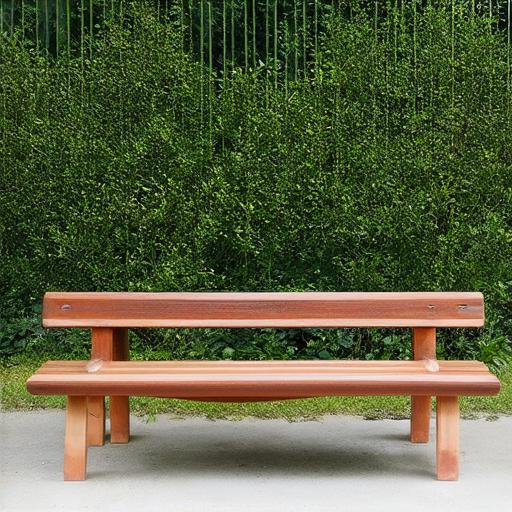} &
    \includegraphics[width=0.8in]{images/cfg_base/01171.jpg} \\
    \rotatebox{90}{\hspace{0.15in}DiffusionNFT} &
    \includegraphics[width=0.8in]{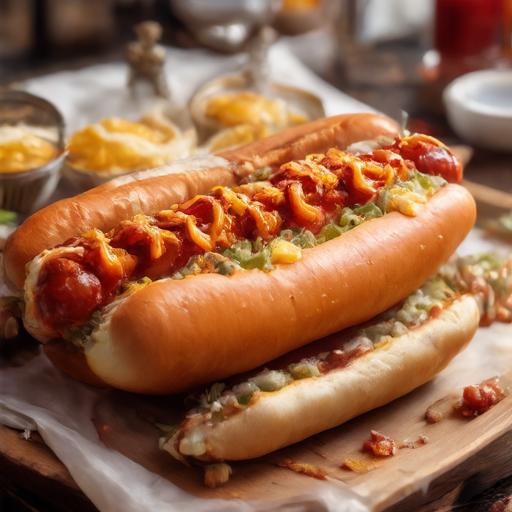} &
    \includegraphics[width=0.8in]{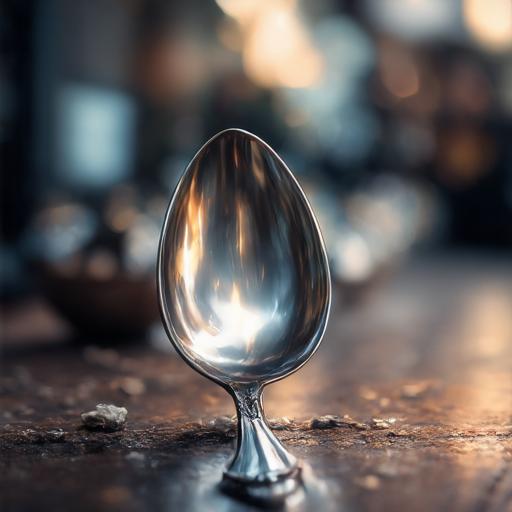} &
    \includegraphics[width=0.8in]{images/nft/01572.jpg} &
    \includegraphics[width=0.8in]{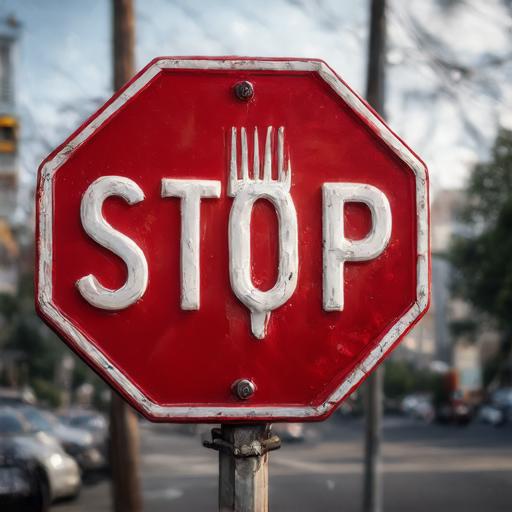} &
    \includegraphics[width=0.8in]{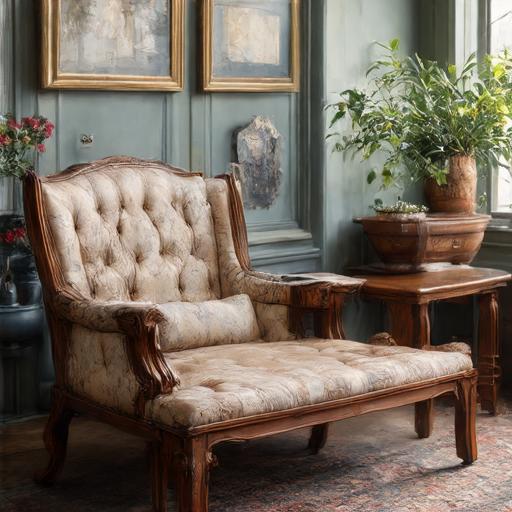} &
    \includegraphics[width=0.8in]{images/nft/01171.jpg} \\
    \rotatebox{90}{\color{red}AdvantageFlow (ours)} &
    \includegraphics[width=0.8in]{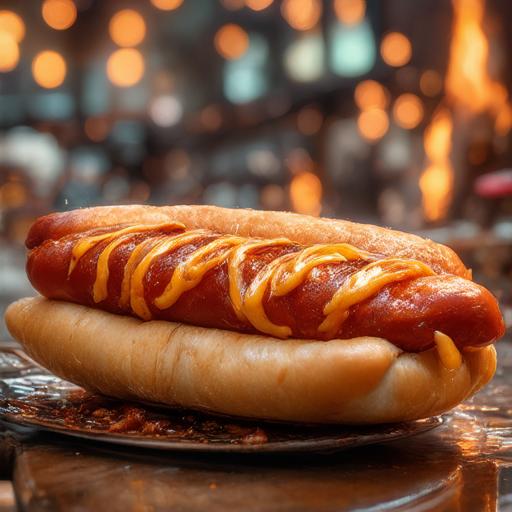} &
    \includegraphics[width=0.8in]{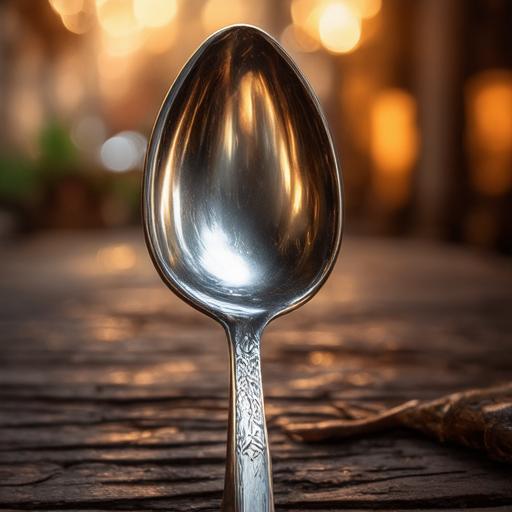} &
    \includegraphics[width=0.8in]{images/af/01572.jpg} &
    \includegraphics[width=0.8in]{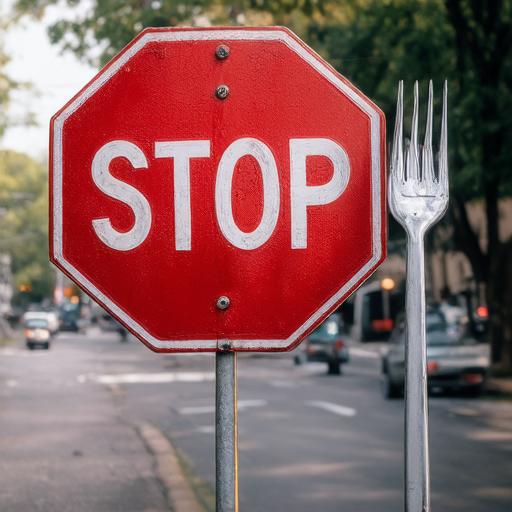} &
    \includegraphics[width=0.8in]{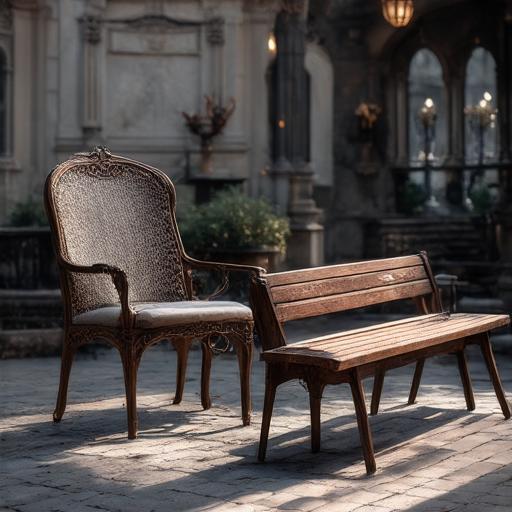} &
    \includegraphics[width=0.8in]{images/af/01171.jpg}
  \end{tabular}} \vspace{0.1in}
  \caption{Images generated by \advantageflow, \diffusionnft, and base model (Stable Diffusion 3.5 Medium) with cfg. We show improvements in object generation.}
  \label{fig:teaser generation}
\end{figure}

\begin{figure}[t!]
  \centering
  {\tiny
  \begin{tabular}{@{\ }rc@{\ }c@{\ }c@{\ }c@{\ }c@{\ }c@{\ }}
    & two cars &
    three apples &
    four zebras &
    bird left of &
    bus above &
    dog above \\
    & &
    &
    &
    couch &
    boat &
    cow \\
    \rotatebox{90}{\hspace{0.1in}Base model + cfg} &
    \includegraphics[width=0.8in]{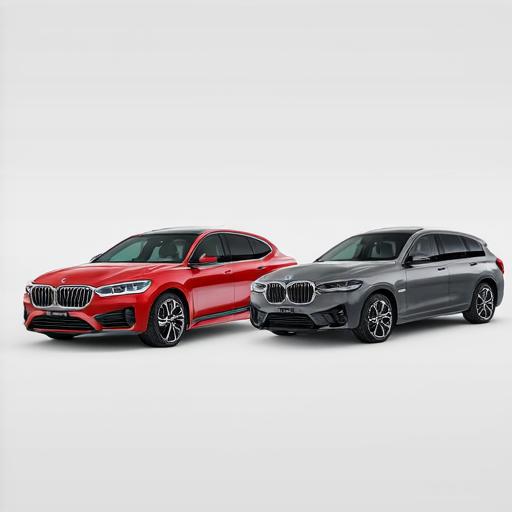} &
    \includegraphics[width=0.8in]{images/cfg_base/02154.jpg} &
    \includegraphics[width=0.8in]{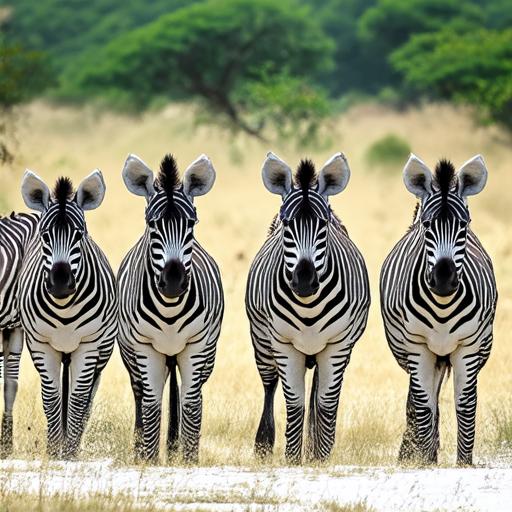} &
    \includegraphics[width=0.8in]{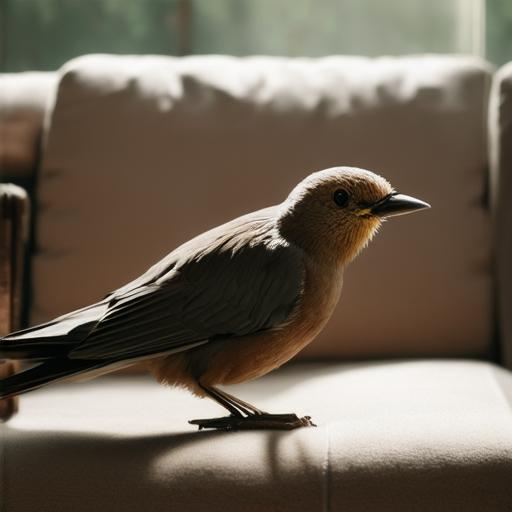} &
    \includegraphics[width=0.8in]{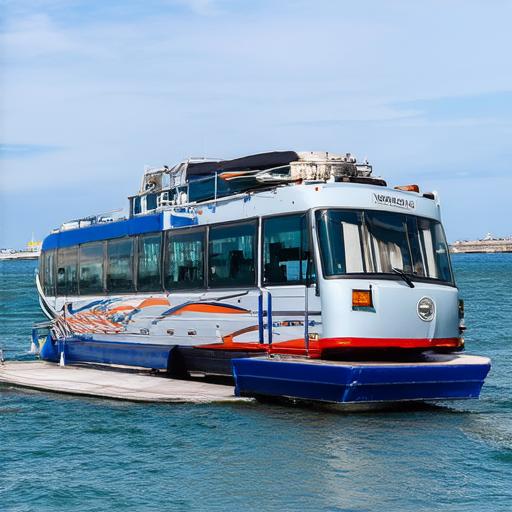} &
    \includegraphics[width=0.8in]{images/cfg_base/01824.jpg} \\
    \rotatebox{90}{\hspace{0.15in}DiffusionNFT} &
    \includegraphics[width=0.8in]{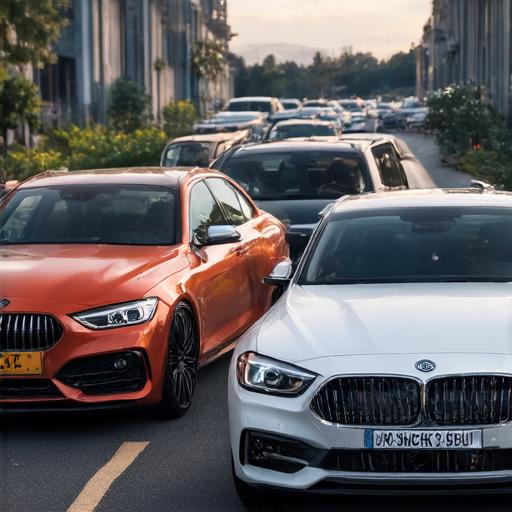} &
    \includegraphics[width=0.8in]{images/nft/02154.jpg} &
    \includegraphics[width=0.8in]{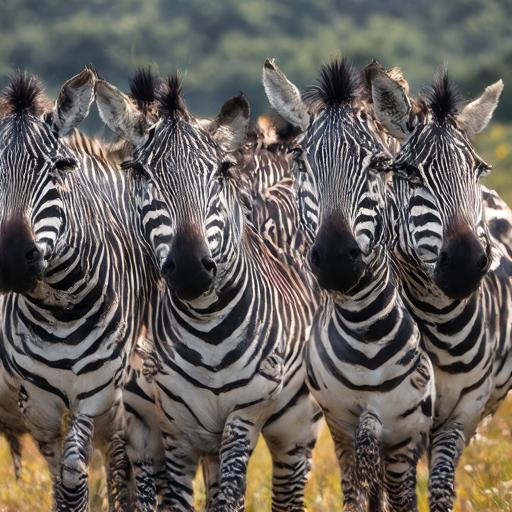} &
    \includegraphics[width=0.8in]{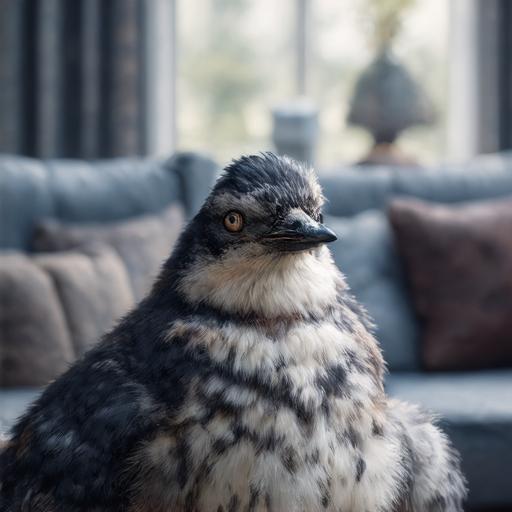} &
    \includegraphics[width=0.8in]{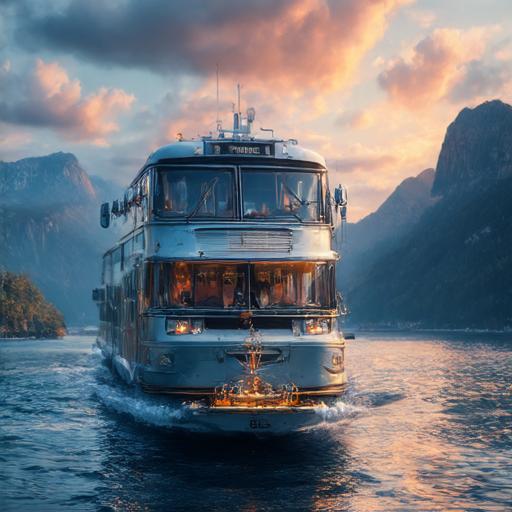} &
    \includegraphics[width=0.8in]{images/nft/01824.jpg} \\
    \rotatebox{90}{\color{red}AdvantageFlow (ours)} &
    \includegraphics[width=0.8in]{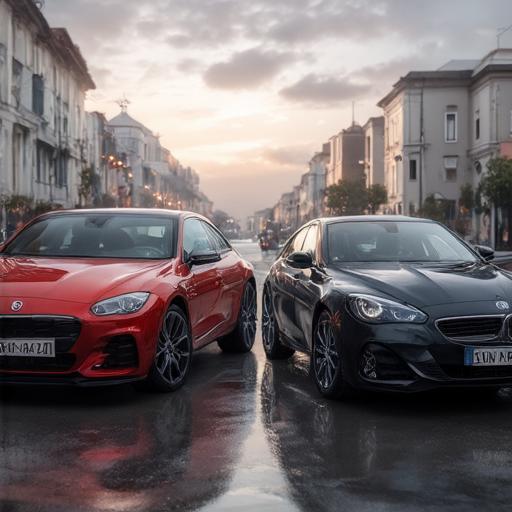} &
    \includegraphics[width=0.8in]{images/af/02154.jpg} &
    \includegraphics[width=0.8in]{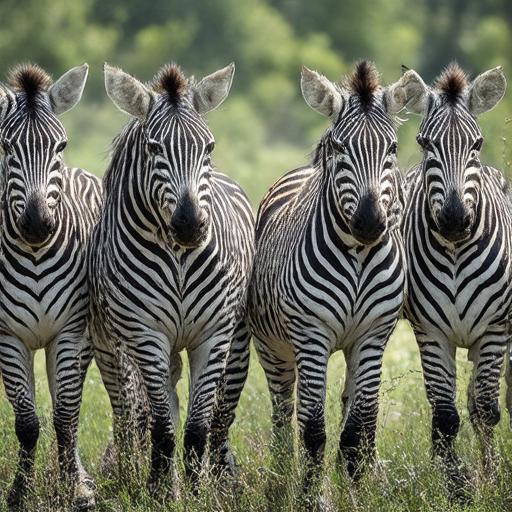} &
    \includegraphics[width=0.8in]{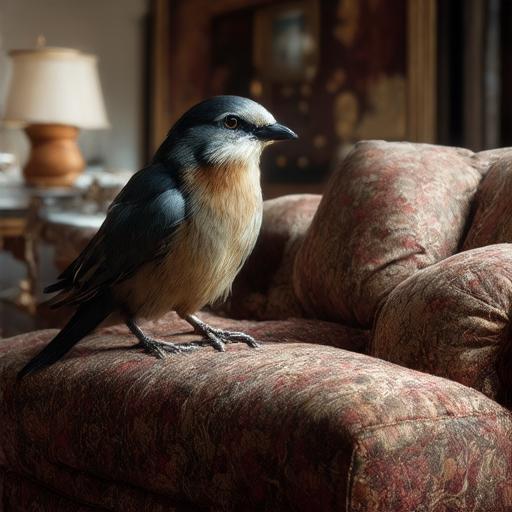} &
    \includegraphics[width=0.8in]{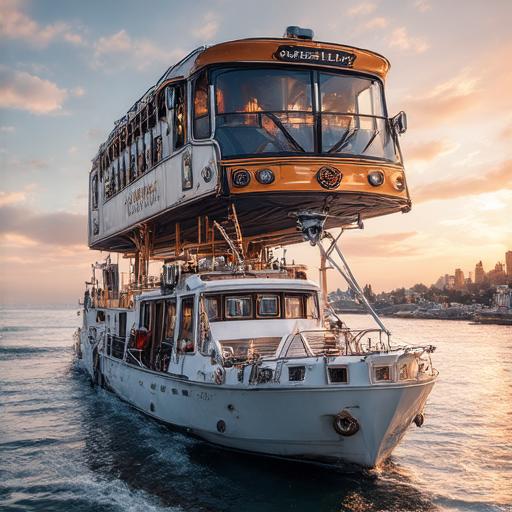} &
    \includegraphics[width=0.8in]{images/af/01824.jpg}
  \end{tabular}} \vspace{0.1in}
  \caption{Images generated by \advantageflow, \diffusionnft, and base model (Stable Diffusion 3.5 Medium) with cfg. We show improvements in counting and position understanding.}
  \label{fig:teaser composition}
\end{figure}

\begin{figure}[t!]
  \centering
  {\tiny
  \begin{tabular}{@{\ }rc@{\ }c@{\ }c@{\ }c@{\ }c@{\ }c@{\ }}
    & white scissors &
    reg giraffe &
    yellow broccoli &
    yellow car and &
    orange skateboard and &
    blue handbag and \\
    & &
    &
    &
    orange toothbrush &
    pink bowl &
    white cell phone \\
    \rotatebox{90}{\hspace{0.1in}Base model + cfg} &
    \includegraphics[width=0.8in]{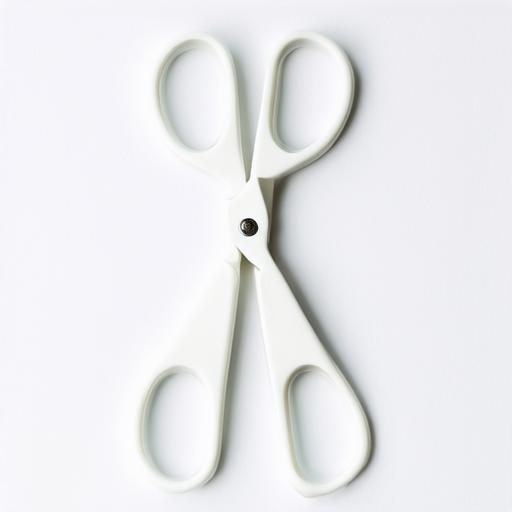} &
    \includegraphics[width=0.8in]{images/cfg_base/02108.jpg} &
    \includegraphics[width=0.8in]{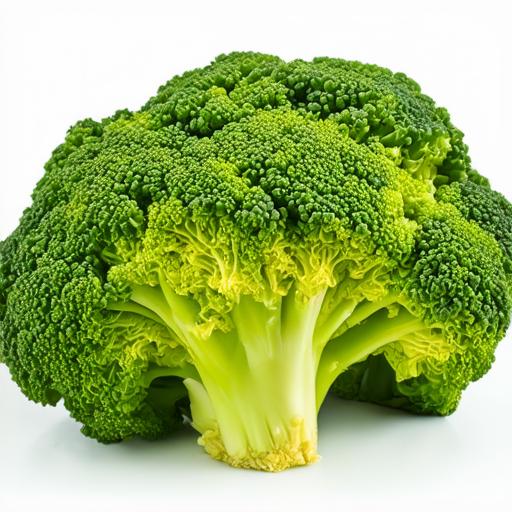} &
    \includegraphics[width=0.8in]{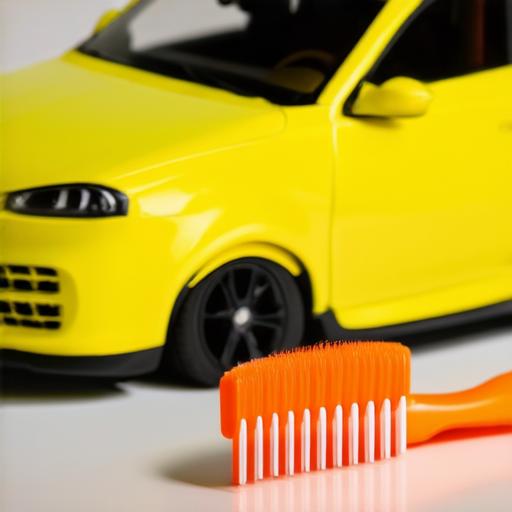} &
    \includegraphics[width=0.8in]{images/cfg_base/01007.jpg} &
    \includegraphics[width=0.8in]{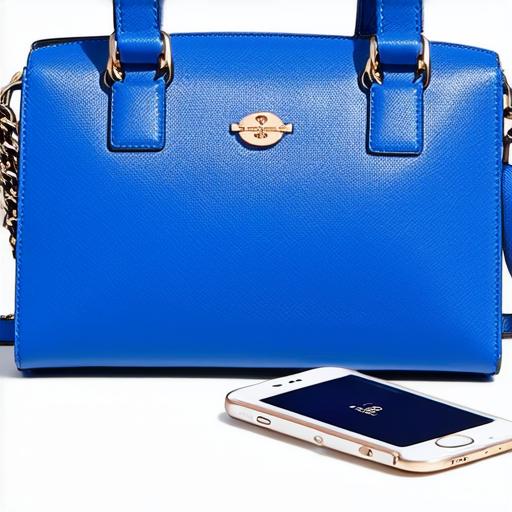} \\
    \rotatebox{90}{\hspace{0.15in}DiffusionNFT} &
    \includegraphics[width=0.8in]{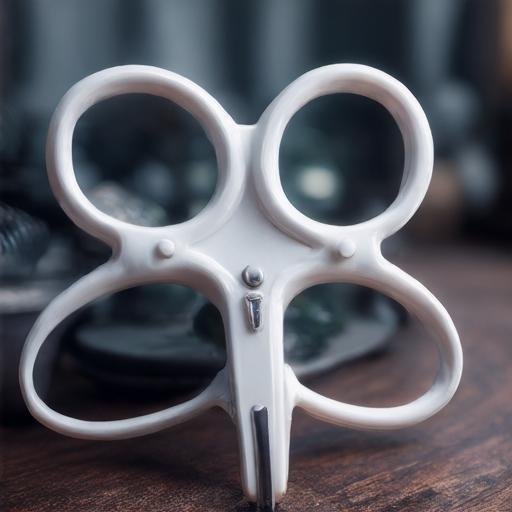} &
    \includegraphics[width=0.8in]{images/nft/02108.jpg} &
    \includegraphics[width=0.8in]{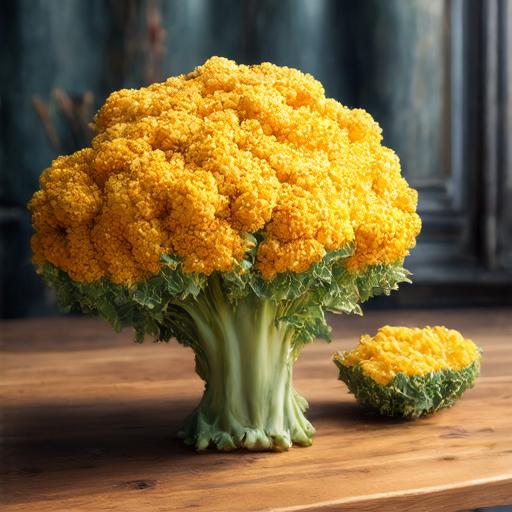} &
    \includegraphics[width=0.8in]{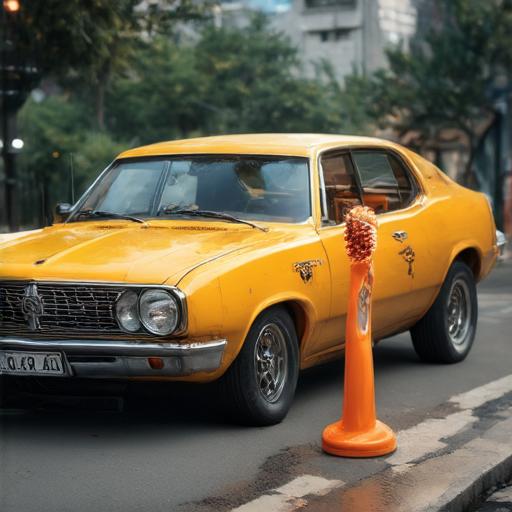} &
    \includegraphics[width=0.8in]{images/nft/01007.jpg} &
    \includegraphics[width=0.8in]{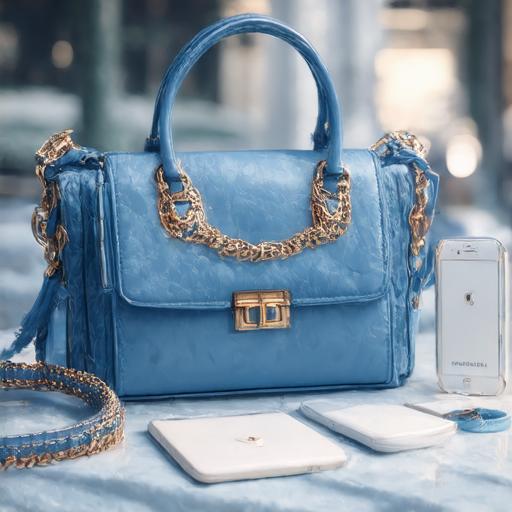} \\
    \rotatebox{90}{\color{red}AdvantageFlow (ours)} &
    \includegraphics[width=0.8in]{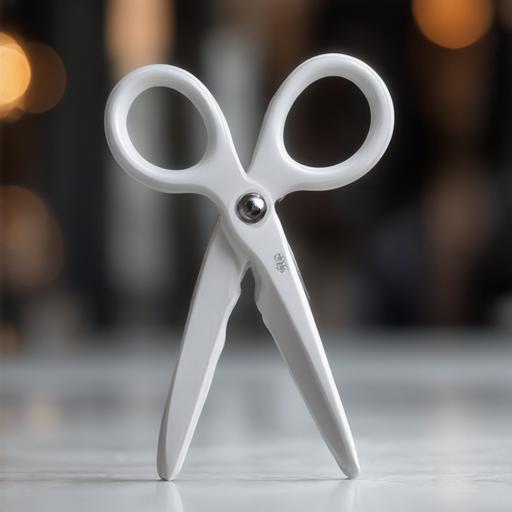} &
    \includegraphics[width=0.8in]{images/af/02108.jpg} &
    \includegraphics[width=0.8in]{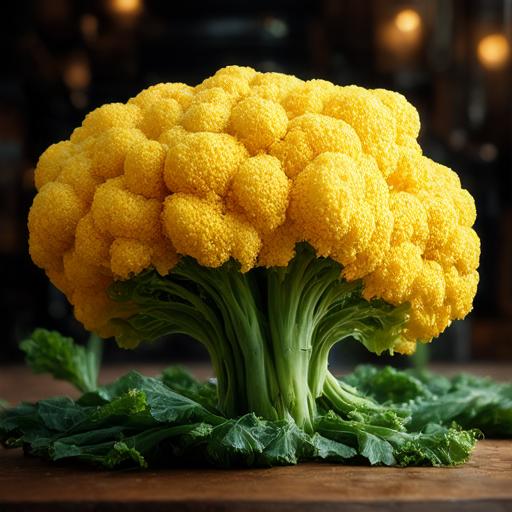} &
    \includegraphics[width=0.8in]{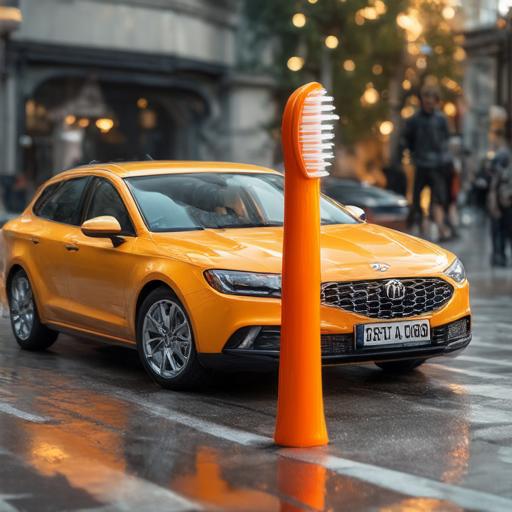} &
    \includegraphics[width=0.8in]{images/af/01007.jpg} &
    \includegraphics[width=0.8in]{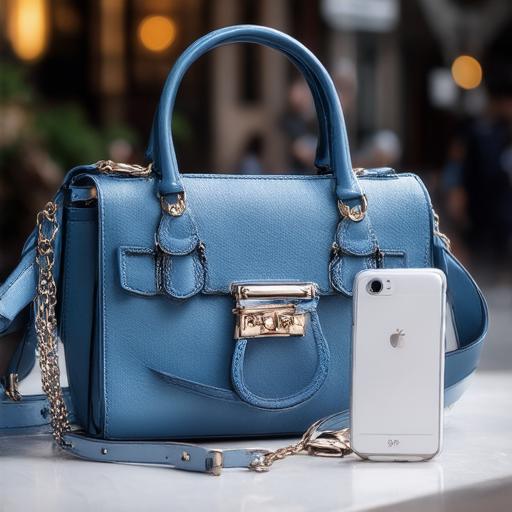}
  \end{tabular}} \vspace{0.1in}
  \caption{Images generated by \advantageflow, \diffusionnft, and base model (Stable Diffusion 3.5 Medium) with cfg. We show improvements in color understanding.}
  \label{fig:teaser color}
\end{figure}

\begin{figure}[t!]
  \centering
  \includegraphics[width=5.6in]{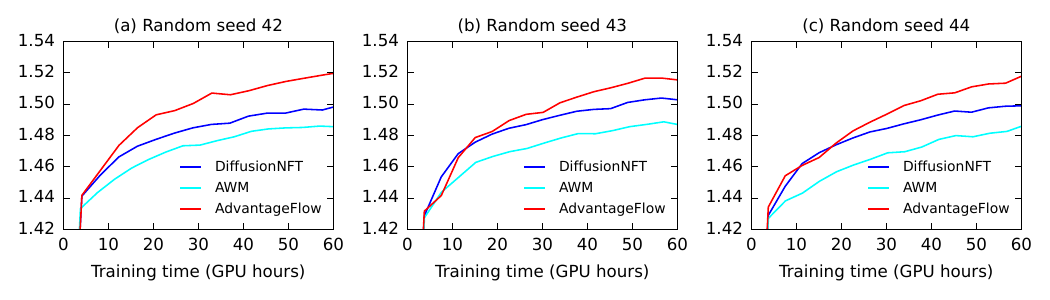} \\ \vspace{-0.1in}
  \caption{Evaluation on multi-objective optimization with random seeds $42$, $43$, and $44$.}
  \label{fig:random seeds}
\end{figure}

Due to the high computational cost of our experiments, and similarly to \citet{diffusionnft} and \citet{xue2025awm}, all our experiments are with a single random seed $42$. To show sensitivity to multiple random seeds, we repeat the experiment in \cref{fig:main}e with seeds $43$ and $44$. Our results are reported in \cref{fig:random seeds}. While the plots change with the seed, the overall conclusion does not: \advantageflow outperforms \diffusionnft, which in turn outperforms \awm.

\section{Fisher--Rao Derivation}
\label{app:fisher-rao}

We prove \cref{prop:fisher-rao}, restated below for convenience.

\paragraph{\cref{prop:fisher-rao} (restated).}
Fix a prompt $c$ and let $F(p) = \E_p[r(\cdot,c)]$.
Under the Fisher--Rao metric, the natural-gradient direction of $F$ at $p = \pold(\cdot\mid c)$ is
\[
  \delta p(x_0) = A(x_0,c)\,\pold(x_0\mid c),\qquad
  A(x_0,c) = r(x_0,c) - \E_{\pold(\cdot\mid c)}[r(\cdot,c)].
\]
A first-order step of size $\eta$ yields the additive tilt
$\qeta(\cdot\mid c) = (1+\eta A(\cdot,c))\,\pold(\cdot\mid c)$ from \cref{eq:additive target}.

\begin{proof}
Let
\begin{equation}
\label{eq:fr-objective}
  F(p) = \E_p[r].
\end{equation}
For a perturbation $\delta p$ of a density $p$, the probability constraint requires
\begin{equation}
\label{eq:probability-tangent-constraint}
  \int \delta p(x)\,\dif x = 0.
\end{equation}
The first variation of $F$ is
\begin{equation}
\label{eq:first-variation}
  \delta F = \int r(x)\,\delta p(x)\,\dif x.
\end{equation}
Under the Fisher--Rao metric, tangent vectors are written as $\delta p(x)=u(x)p(x)$ with $\E_p[u]=0$, and the inner product is
\begin{equation}
\label{eq:fisher-rao-inner-product}
  \inner{u}{v}_p = \E_p[u(x)v(x)].
\end{equation}
The gradient direction $u^*$ must satisfy
\begin{equation}
\label{eq:fisher-rao-gradient-condition}
  \E_p[u^*(x)u(x)] = \E_p[r(x)u(x)]
\end{equation}
for every zero-mean tangent function $u$. Therefore
\begin{equation}
\label{eq:fisher-rao-gradient}
  u^*(x) = r(x) - \E_p[r].
\end{equation}
At $p=\pold$, this gives the Fisher--Rao natural-gradient direction
\begin{equation}
\label{eq:fisher-rao-natural-gradient-direction}
  \delta p(x) = A(x)\,\pold(x),
  \qquad
  A(x) = r(x)-\E_{\pold}[r].
\end{equation}
A first-order step of size $\eta$ gives
\begin{equation}
\label{eq:fisher-rao-additive-step}
  \qeta(x)
  = \pold(x) + \eta A(x)\pold(x)
  = \bigl(1+\eta A(x)\bigr)\pold(x).
\end{equation}
\end{proof}

\section{Two-Anchor Equivalence and Variance Reduction}
\label{app:two-anchor-equivalence}

We prove \cref{prop:two-anchor}, restated below for convenience.

\paragraph{\cref{prop:two-anchor} (restated).}
Fix a prompt $c$ and suppose $\fold(x_t,t,c)=\E_{x_0\sim\pold(\cdot\mid c),\,\epsilon}[x_0\mid x_t,t,c]$.
Up to a $\theta$-independent constant,
\[
  \E_{x_0\sim\pold,\,\epsilon,\,t}\bigl[\norm{x_0 - \ftheta(x_t,t,c)}^2\bigr]
  = \E_{x_0\sim\pold,\,\epsilon,\,t}\bigl[\norm{\fold(x_t,t,c) - \ftheta(x_t,t,c)}^2\bigr].
\]
Moreover, the single-sample gradient of the right-hand side has no larger variance than that of the left-hand side.

\begin{proof}
Fix a prompt $c$ and suppress it for brevity. Fitting the additive target gives
\begin{equation}
\label{eq:appendix-qeta-loss}
  \cL_{\qeta}(\theta)
  = \E_{\pold,\,\epsilon,\,t}
    \bigl[(1+\eta A(x_0))\,\norm{x_0-\ftheta(x_t,t)}^2\bigr].
\end{equation}
Expanding,
\begin{equation}
\label{eq:appendix-qeta-loss-expanded}
\begin{aligned}
  \cL_{\qeta}(\theta)
  &= \E_{\pold,\,\epsilon,\,t}
    \bigl[\norm{x_0-\ftheta(x_t,t)}^2\bigr] \\
  &\quad + \eta\,\E_{\pold,\,\epsilon,\,t}
    \bigl[A(x_0)\,\norm{x_0-\ftheta(x_t,t)}^2\bigr].
\end{aligned}
\end{equation}
Assume
\begin{equation}
\label{eq:old-bayes-predictor}
  \fold(x_t,t) = \E_{\pold}[x_0 \mid x_t,t].
\end{equation}
Then
\begin{equation}
\label{eq:target-decomposition}
  x_0-\ftheta = (x_0-\fold) + (\fold-\ftheta).
\end{equation}
Taking squared norms and expectations,
\begin{equation}
\label{eq:squared-decomposition}
\begin{aligned}
  \E\norm{x_0-\ftheta}^2
  &= \E\norm{x_0-\fold}^2
   + \E\norm{\fold-\ftheta}^2 \\
  &\quad + 2\,\E\inner{x_0-\fold}{\fold-\ftheta}.
\end{aligned}
\end{equation}
The cross term vanishes because
\begin{equation}
\label{eq:cross-term-zero}
  \E[x_0-\fold \mid x_t,t] = 0.
\end{equation}
Therefore
\begin{equation}
\label{eq:sample-to-anchor-equivalence}
  \E_{\pold,\,\epsilon,\,t}
  \bigl[\norm{x_0-\ftheta(x_t,t)}^2\bigr]
  = \E_{\pold,\,\epsilon,\,t}
    \bigl[\norm{\fold(x_t,t)-\ftheta(x_t,t)}^2\bigr] + C,
\end{equation}
where $C$ is independent of $\theta$. Substituting and rescaling by $1/\eta$ gives
\begin{equation}
\label{eq:appendix-afm-loss}
  \lossAFM(\theta)
  = \E_{\pold,\,\epsilon,\,t}
    \left[
      A(x_0)\,\norm{x_0-\ftheta(x_t,t)}^2
      + \frac{1}{\eta}\,\norm{\ftheta(x_t,t)-\fold(x_t,t)}^2
    \right],
\end{equation}
up to a $\theta$-independent constant.

\paragraph{Variance reduction.}
Let
\[
  g_{\mathrm{sample}} = \nabla_\theta\,\norm{x_0-\ftheta(x_t,t)}^2,\qquad
  g_{\mathrm{rollout}} = \nabla_\theta\,\norm{\fold(x_t,t)-\ftheta(x_t,t)}^2
\]
be the single-sample gradients of the two terms. Since $\fold(x_t,t)=\E[x_0\mid x_t,t]$,
\begin{equation}
\label{eq:rao-blackwell-gradient}
  g_{\mathrm{rollout}}
  = \E[g_{\mathrm{sample}} \mid x_t,t].
\end{equation}
The Rao--Blackwell inequality gives
\begin{equation}
\label{eq:rao-blackwell-variance}
  \Var(g_{\mathrm{rollout}}) \leq \Var(g_{\mathrm{sample}}).
\end{equation}
\end{proof}

\section{DiffusionNFT Algebra}
\label{app:diffusionnft-algebra}

We prove the \diffusionnft{} connection stated in \cref{sec:related work}.

\begin{proposition}[\diffusionnft{} as a special case of \advantageflow]
\label{prop:diffusionnft}
Let $r(x_0, c) \in[0,1]$ be a normalized optimality score and $A(x_0, c)=2r(x_0, c)-1$.
The \diffusionnft{} branch loss, after replacing velocity losses with prediction losses and dropping $\theta$-independent constants, equals
\[
  \beta A(x_0,c)\,\norm{x_0-\ftheta(x_t, t, c)}^2 + \beta(\beta-A(x_0,c))\,\norm{\ftheta(x_t, t, c)-\fold(x_t, t, c)}^2,
\]
corresponding to the \advantageflow{} objective \eqref{eq:advantageflow loss} with $\lambda=0$ and
rollout regularization schedule $\gammaNFT(A)=\beta(\beta-A(x_0, c))$.
\end{proposition}

\begin{proof}
We work in velocity space, since that is how \diffusionnft{} states its objective.

Let $v$ be the forward-process velocity target, let $\vold$ be the rollout model, and let $\vtheta$ be the trainable model. \diffusionnft{} defines positive and negative branches
\begin{equation}
\label{eq:nft-branches}
  \vtheta^+ = (1-\beta)\vold + \beta\vtheta,
  \qquad
  \vtheta^- = (1+\beta)\vold - \beta\vtheta.
\end{equation}
Let $r\in[0,1]$ be the normalized optimality score and define $A(x_0, c)=2r(x_0, c)-1$. The branch loss is
\begin{equation}
\label{eq:nft-branch-loss}
  \ellNFT
  = r\,\norm{\vtheta^+-v}^2
    + (1-r)\,\norm{\vtheta^- - v}^2.
\end{equation}
Set $d=\vtheta-\vold$ and $e=\vold-v$. Then
\begin{equation}
\label{eq:nft-d-e}
  \vtheta^+ - v = e+\beta d,
  \qquad
  \vtheta^- - v = e-\beta d.
\end{equation}
Expanding,
\begin{equation}
\label{eq:nft-branch-expanded}
  \ellNFT
  = \norm{e}^2 + \beta^2\norm{d}^2 + 2\beta A(x_0, c)\inner{e}{d}.
\end{equation}
Now consider the two-anchor expression
\begin{equation}
\label{eq:nft-two-anchor-velocity}
  A(x_0, c)\,\norm{\vtheta-v}^2 + (\beta-A(x_0, c))\,\norm{\vtheta-\vold}^2.
\end{equation}
Since $\vtheta-v=e+d$, this equals
\begin{equation}
\label{eq:nft-two-anchor-expanded}
  A(x_0, c)\,\norm{e}^2 + \beta\,\norm{d}^2 + 2A(x_0, c)\inner{e}{d}.
\end{equation}
Multiplying by $\beta$ gives
\begin{equation}
\label{eq:nft-two-anchor-rescaled}
  \beta A(x_0, c)\,\norm{e}^2 + \beta^2\norm{d}^2 + 2\beta A(x_0, c)\inner{e}{d}.
\end{equation}
This differs from \cref{eq:nft-branch-expanded} only by $(1-\beta A(x_0, c))\norm{e}^2$, which is independent of $\vtheta$. Therefore, after dropping constants,
\begin{equation}
\label{eq:nft-equivalence}
  \ellNFT
  \equiv
  \beta A(x_0, c)\,\norm{\vtheta-v}^2 + \beta(\beta-A(x_0, c))\,\norm{\vtheta-\vold}^2.
\end{equation}
To pass from velocity space to prediction space, use
\[
  f_\theta(x_t,t,c)=x_t-t\vtheta(x_t,t,c),
  \qquad
  \fold(x_t,t,c)=x_t-t\vold(x_t,t,c),
\]
and let \(v=\epsilon-x_0\) be the forward-process velocity target. Then
\[
  x_0-\ftheta(x_t,t,c)
  =
  t(\vtheta-v),
  \qquad
  \ftheta(x_t,t,c)-\fold(x_t,t,c)
  =
  -t(\vtheta-\vold).
\]
Thus, after moving to the prediction-loss by multiplying the
velocity-space objective by \(t^2\), we have
\[
  t^2\norm{\vtheta-v}^2
  =
  \norm{x_0-\ftheta(x_t,t,c)}^2,
  \qquad
  t^2\norm{\vtheta-\vold}^2
  =
  \norm{\ftheta(x_t,t,c)-\fold(x_t,t,c)}^2 .
\]
Substituting these identities into the velocity-space expression gives the
prediction-space form of \cref{eq:nft-equivalence}. Thus \diffusionnft{} corresponds to the advantage-dependent schedule
\begin{equation}
\label{eq:nft-gamma-schedule-appendix}
  \gammaNFT(A) = \beta(\beta - A(x_0, c)).
\end{equation}
For $\beta=1$, this becomes $\gammaNFT(A)=1-A(x_0, c)$.
\end{proof}

\end{document}